\documentclass[11pt, a4paper, logo, copyright]{deepmind}

\usepackage[all]{xy}

\usepackage{lipsum}

\usepackage[numbers]{natbib}

\usepackage{booktabs}

\usepackage{import}
\graphicspath{{figures/}}

\usepackage{epigraph}

\usepackage{enumitem}
\setlist[itemize]{topsep=0pt}

\usepackage[dvipsnames]{xcolor}

\usepackage{tikz}
\usetikzlibrary{trees}

\DeclareMathOperator*{\argmin}{arg\,min}
\DeclareMathOperator*{\argmax}{arg\,max}

\usepackage{amsmath}
\usepackage{mathtools}
\usepackage{amsthm}
\usepackage{algorithm}
\usepackage{algpseudocode}

\newtheorem{proposition}{Proposition}
\newtheorem{definition}{Definition}

\newtheorem{zrem}{Remark}
\newenvironment{rem}{\par\small\zrem}{\endzrem}

\newcommand{\red}[1]{{\color{Red} #1}}
\newcommand{\green}[1]{{\color{Green} #1}}
\newcommand{\blue}[1]{{\color{Blue} #1}}
\newcommand{\white}[1]{{\color{Gray} #1}}
\newcommand{\yellow}[1]{{\color{YellowOrange} #1}}

\title{Beyond Bayes-optimality: meta-learning what you know you don't know}

\correspondingauthor{jordigrau@deepmind.com}

\keywords{Bayesian agents, risk, ambiguity, meta-learning, AI safety}
\paperurl{Technical Report}

\reportnumber{001} 

\author[1*]{Jordi Grau-Moya}
\author[1*]{Gr\'egoire Del\'etang}
\author[1*]{Markus Kunesch}
\author[1]{Tim Genewein}
\author[1]{Elliot Catt}
\author[1]{Kevin Li}
\author[1]{Anian Ruoss}
\author[2]{Chris Cundy}
\author[1]{Joel Veness}
\author[1]{Jane Wang}
\author[1]{Marcus Hutter}
\author[1]{Christopher Summerfield}
\author[1]{Shane Legg}
\author[1]{Pedro Ortega}

\affil[*]{Equal contribution}
\affil[1]{DeepMind, London}
\affil[2]{Department of Computer Science,
Stanford University}

\begin{abstract}
Meta-training agents with memory has been shown to culminate in Bayes-optimal agents, which casts Bayes-optimality as the implicit solution to a numerical optimization problem rather than an explicit modeling assumption. Bayes-optimal agents are risk-neutral, since they solely attune to the expected return, and ambiguity-neutral, since they act in new situations as if the uncertainty were known. This is in contrast to risk-sensitive agents, which additionally exploit the higher-order moments of the return, and ambiguity-sensitive agents, which act differently when recognizing situations in which they lack knowledge. Humans are also known to be averse to ambiguity and sensitive to risk in ways that aren't Bayes-optimal, indicating that such sensitivity can confer advantages, especially in safety-critical situations. How can we extend the meta-learning protocol to generate risk- and ambiguity-sensitive agents?  
The goal of this work is to fill this gap in the literature  by showing that risk- and ambiguity-sensitivity also emerge as the result of an optimization problem using modified meta-training algorithms, which manipulate the experience-generation process of the learner.  We empirically test our proposed meta-training algorithms on agents exposed to foundational classes of decision-making experiments and demonstrate that they become sensitive to risk and ambiguity.
\end{abstract}

\begin{document}

\maketitle

\section{Introduction}

Reasoning about uncertainty is a hallmark of human intelligence. In artificial intelligence, handling uncertainty is critical for developing systems that act rationally and safely. Such systems should know when to be cautious and avoid catastrophic events, and when to take risks to reap greater benefits. 

But uncertainty comes in different flavors. The economist Frank H. Knight~\cite{Knight1921} proposed a subtle but important distinction between two fundamental types of uncertainty in decision-making, \emph{risk} and \emph{ambiguity}. Risk (i.e. known or closed-world uncertainty) applies to familiar situations where the exact outcome of an event is uncertain but probabilities can be computed, as in roulette wheels and dice. Knowing the probabilities over outcomes enables reducing randomness to effective certainty (e.g.\ via computing certainty equivalents). In contrast, ambiguity (i.e. unknown or open-world uncertainty) refers to the uncertainty in unfamiliar situations where the probabilities are not known or cannot be determined, e.g.\ whether in six years the US president will be a Democrat, or the answer to the question ``are Cydophines also Abordites?'' before attending a talk that might explain what these terms mean (the example discussed in~\cite{gilboa2009always}). 

In a seminal paper, Daniel Ellsberg argued that humans are acutely sensitive to risk and ambiguity~\cite{ellsberg1961risk}. This is likely because it could be advantageous to use different decision mechanisms in familiar but uncertain (risky) situations compared to truly novel (ambiguous) situations. Indeed, it has been found that risk and ambiguity are represented by distinct patterns of neural activity in the human brain~\cite{huettel2006neural, hsu2005neural}, and that patients with damage to a higher order brain structure called the orbitofrontal cortex are more risk- and ambiguity-neutral compared to control subjects~\cite{hsu2005neural}. This indicates that being sensitive to or aware of not only the risks, but also \textit{what is unknown} in a given situation can confer evolutionary advantages, and has even been observed in non-human primates~\cite{rosati2011chimpanzees}. This leads us to the questions: 1) Under what circumstances might being risk- or ambiguity-sensitive beneficial? and 2) How can we train agents to display such biases?

Prior work revealed that an agent equipped with memory \cite{wang2016learning,duan2016rl} can meta-learn a Bayes-optimal policy, i.e.\ a policy that (implicitly) reasons about uncertain hypotheses, optimally trading off exploration versus exploitation \cite{ortega2019meta,mikulik2020meta}. This casts Bayes-optimality as the implicit solution to a numerical optimization problem, rather than an explicit \emph{a priori} modeling requirement. While Bayes-optimal policies use uncertainty for guiding their behavior, they do so in a restricted form in its strict definition: first, they are \emph{risk-neutral}, i.e.\ insensitive to the shape of the distribution over returns except for the expected value; and second, they are also \emph{ambiguity-neutral}, that is, acting as if the uncertainty were known.

In this work we demonstrate how to adapt the meta-learning protocol to build agents sensitive to risk and ambiguity. The key proposal is a modification to the \emph{experience-generation process} of the learner. In particular, we argue that risk-sensitivity ensues when the environment's responses appear to anticipate the agent's actions. For example, an agent becomes risk-seeking at test time if during training there is another agent in the environment that anticipates the agent's plans and intervenes to make them more likely to succeed. Furthermore, we demonstrate that ambiguity-sensitivity arises if the agent can experience and exploit novelty via an ensemble. An example of this could be an agent that receives advice from members of a committee, who would naturally disagree in novel situations, and learns from experience how best to deal with conflicting advice. As in standard meta-learning, the resulting policies are the solutions to a numerical optimization problem, with the crucial difference that our protocols generate policies that are not Bayes-optimal. Instead, they are sensitive to the higher-order moments of the return (in the case of risk) or they detect missing information or unknown probabilities, and decide accordingly (in the case of ambiguity).

Would winning the lottery multiple times make you more risk-seeking when gambling? Would having a cycle accident make you more risk-averse when driving? These adaptations are hard to explain when risk- and ambiguity-sensitivity are formulated as a rigid cognitive trait of the agent or as part of its decision-making principle. However, describing risk and ambiguity as a property of the data source offers great explanatory power. It operationalizes their differences and suggests concrete training protocols with minimal cognitive requirements from the agent. And, it is a natural explanation for context-dependent risk- and ambiguity-sensitivity, i.e. the ability of an agent to adjust its sensitivity depending on the situation based on experience, and perhaps more accurately reflects adaptive decision-making in the real world, which is always operating under incomplete information.

The paper is organized as follows. Section~\ref{sec:related-work} puts our work in context with the machine learning and economics literature. In Section~\ref{sec:bayes-risk-and-ambiguity-an-example} we set the stage by explaining the distinction between Bayesian, risk-sensitive, and ambiguity-sensitive agents with the help of an illustrative example. 
Section~\ref{sec:meta-learning-risk-ambiguity} describes the fundamental ideas on how to adapt meta-learning to build agents with risk-sensitivity and ambiguity-sensitivity. Moving on to an experimental evaluation, Section~\ref{sec:experimental-methodology} describes the particular implementation that we used to build our agents, and the details of the environments that we used in our experiments. Section~\ref{sec:experimental-results} focuses on the experimental results confirming the validity of our meta-training procedures. We end with a discussion (Section~\ref{sec:discussion}) and final conclusion.

\section{Related work}\label{sec:related-work}

Risk-sensitivity has been extensively studied in the machine learning and reinforcement learning communities. The main driver for this line of research is to build agents that are safer and more robust to external perturbations by being sensitive to the variability of the reward. The scope is wide, ranging from formulations concerned with the bandit and Markov decision process (MDP) setting~\cite{cassel2018general,howard1972risk,nilim2005robust,tamar2014scaling} to models focusing on the RL setting~\cite{mihatsch2002risk,deletang2021model,fei2020risksensitive} where the dynamics model and reward function are unknown. Some models heavily rely on variance alone to quantify risk, while others can also capture higher order moments~\cite{chow2015risk,fei2020risksensitive} via entropic risk measures for example. Recent research on distributional RL builds directly a distributional Bellman operator as the main tool for learning distribution over returns ~\cite{stanko2019risk,nuria2021riskaverse} which they could be used for learning risk-sensitive policies.

The detection of uncertainty and out-of-distribution data is an important aspect of safety \cite{amodei2016concrete,mcallister2017concrete} and has been studied extensively in the literature \cite{osband2021epistemic}. Approaches include MC-dropout \cite{gal2016dropout}, which for example was used to reduce the velocity of robots in uncertain situations \cite{kahn2017uncertainty}, and deep ensembles \cite{osband2016deep,lakshminarayanan2017simple,tifrea2022semi}. Our contribution on detecting ambiguity uses standard tools such as ensembles~\cite{osband2021epistemic}, however, the novelty lies in using the ensemble output to meta-train an agent that is able to handle ambiguous situations.

The economic literature contains extensive efforts on modeling risk and ambiguity.  Risk models can be traced back to at least 1738 with Daniel Bernoulli's solution to the St.\ Petersburg paradox~\cite{bernoulli1738}. The first models relied on the curvature of the utility function~\cite{pratt1964,arrow1965aspects}, followed by risk-return models that consider a expectation-variance trade-off~\cite{markowitz1952} and  higher-order moments about the expectation of utility~\cite{whittle1981risk,pichler2020entropy}. Ambiguity is a relatively younger concept~\cite{ellsberg1961risk} and harder to pin down. The first models largely rely on worst- and best-case extremes to free themselves from uncertain probabilities~\cite{wald1950statistical,arrow1972optimality}. Recent models use multiple priors~\cite{gilboa1989maxmin,klibanoff2005smooth,ghirardato2004differentiating} or uncertainty sets~\cite{etner2012decision,maccheroni2006ambiguity,hansen2011robustness}.

Our approach differs from the existing literature by going beyond hard-coded responses to uncertainty towards more data-dependent descriptions of  risk- and ambiguity-sensitive behavior: it allows the agent to differentiate between risk and ambiguity and to learn an appropriate, context-dependent response to both from experience.

\section{Bayes, Risk, and Ambiguity: An Illustrative Example} \label{sec:bayes-risk-and-ambiguity-an-example}

We have seen the following distinction of uncertainty\footnote{Some authors use different definitions of \emph{uncertainty}, for example \cite{Knight1921} uses it as a synonym for ambiguity. Here, we use uncertainty as an umbrella term that encompasses both risk and ambiguity.}:
\begin{figure}[h!]
  \centering
  \begin{tikzpicture}
    [edge from parent fork down]
    \node {Uncertainty} [sibling distance = 2cm]
      child {node {Risk\vphantom{g}}}
      child {node {Ambiguity}};
  \end{tikzpicture}
\end{figure}
\\Now we elaborate how they differ in practice.
Consider a game in which you are presented with two transparent boxes containing colored marbles, and you are asked to choose one (Figure~\ref{fig:urn-choices}).  Then, a marble is drawn randomly from your chosen box, and you receive a payoff which depends on the marble's color. You know the payoff for some of the colors, but not all. In each case (a, b, c, \& d) of Figure~\ref{fig:urn-choices}, what would your decision be?

\begin{figure}[t!]
  \centering
  \def\svgwidth{0.8\textwidth}
  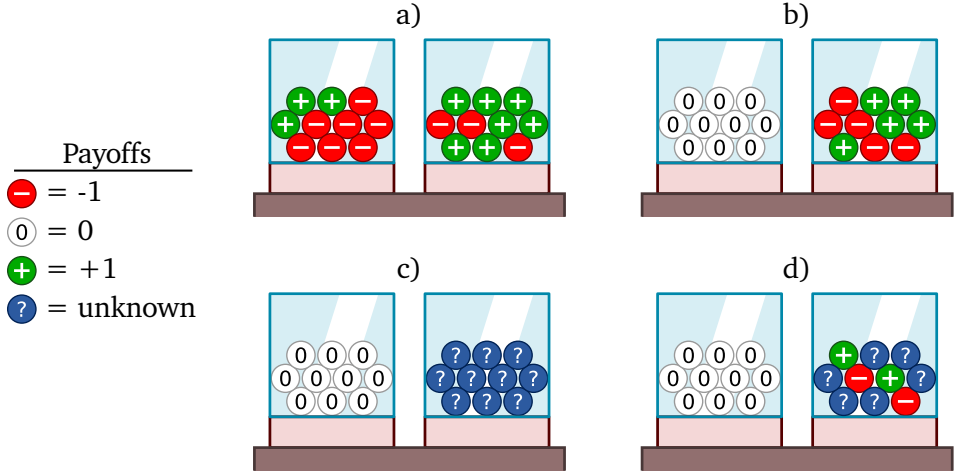
  \caption{Choices with uncertainty. In each case, you are asked to choose either the left or the right box, each containing ten colored marbles in a known proportion. You obtain a payoff corresponding to the color of a marble drawn randomly from the chosen box.}
  \label{fig:urn-choices}
\end{figure}

We next compare the choices an Expected Utility, a Bayes-optimal, a risk-sensitive, and an ambiguity-sensitive agent would make. Intuitively, one can think of an agent as making a choice in two steps: a) assigning a value to each option and b) proposing a (stochastic) decision. By assigning a value, the agent reduces an option with multiple random outcomes to a single representative certain outcome, which is why this reduction function is also known as the \emph{certainty-equivalent} in the economics literature. The second step requires the agent to form a probability distribution over actions, from which a final choice will be drawn.

\paragraph{Notation.} In the rest of this section we adopt the simplified case of an agent issuing a single action $A \sim \pi$  and receiving a single observation $O \sim T(\cdot | a)$ where $a \in \mathcal A := \{\text{`left'}, \text{`right'} \}$ is the realization of $A$. We summarize the generative process with
\begin{equation}
 A \rightarrow  O
\end{equation}
We denote the reward function as $r : \mathcal O \rightarrow \mathbb R$, which maps the observation space $\mathcal O$ to numerical reward values.  

\paragraph{Expected Utility and Bayes-optimality.} A Bayes-optimal agent assigns certainty-equivalents equal to the expected payoff with respect to its beliefs and then chooses the maximizing policy. This conforms to the paradigm of \emph{expected utility} (EU)~ \cite{vonNeumann1947theory,savage1972foundations}, the gold standard in classical microeconomic theory \cite{rubinstein1998modeling} and in reinforcement learning~\cite{sutton2018reinforcement,russell2010artificial,hutter2004universal}. Formally, for our simplified example,  the certainty-equivalent of an EU maximizer is defined as 
\begin{equation}
 Q_{\text{EU}}(a) := \sum_{o\in O} p(o | a) r(o)    
\end{equation}
where $p(o|a)$ plays the role of a subjective prior distribution over observations.  The maximizing policy set is computed as $\Pi^* \in \argmax_{\pi} \sum_{a \in \mathcal A} \pi(a) Q_{\text{EU}}(a)$. Normally, $\Pi^*$ contains a single optimal policy, but in general it also allows for multiple optimal policies. Note that in the case of Bayes-optimal valuation we have
\begin{equation}\label{eq:bayes-simple-value}
 Q_{\text{BO}}(a) := \sum_{\theta \in \Theta } p(\theta) \sum_{o\in O} p_\theta(o | a) r(o)    
\end{equation}
which requires the agent to have a prior $p(\theta)$ over models $\theta \in \Theta$ indexing the observation model $p_\theta$. For now we  focus on how the EU agent makes choices while  leaving  for later analysis the case of Bayes-optimal agent with priors.

An EU agent chooses as follows:
\begin{enumerate}
\item[a)] Since the expected payoffs for the left and right boxes are -0.4 and +0.4 respectively, the agent chooses the right box.
\item[b)] The expected payoffs are both equal to zero, hence the agent is indifferent between the two options: any distribution over the two choices is a valid solution.
\item[c)] Here the agent confronts an ambiguous situation, since the payoffs for the blue marbles in the right box are unknown. The choice of the EU agent is undefined, since there is not enough information available to assign a certainty-equivalent.
\item[d)] As in case (c).
\end{enumerate}

\begin{rem}\label{rem:bayesian-ambiguity}
A reader trained in the expected utility paradigm might object to the conclusions arrived in cases~(c) and~(d). Why can't the agent simply place a (subjective) prior distribution over the payoffs of the blue marble?
This is a fair objection, and we will return to this point later.
\end{rem}

\paragraph{Risk-sensitive.} A risk-sensitive agent is similar to an EU agent, with the crucial difference that the certainty-equivalent is not required to be the expectation, but another function potentially sensitive to the higher-order moments of the payoff (e.g.\ variance, skewness, kurtosis). For simplicity, assume a certainty-equivalent that only accounts for the expectation and  the variance i.e., 
\begin{equation}\label{eq:simple_risk_sensitive_valuation}
    Q_\text{Risk}(a) := \sum_o p(o|a) r(o) + \beta \sum_o p(o|a) \left( r(o) - \bar r(a) \right)^2
\end{equation} 
where $\bar r(a) := \sum_o p(o|a) r(o)$ is the mean reward. Hence a risk-sensitive agent does \emph{not} conform to the expected utility paradigm. If the certainty-equivalent is lower or higher than the expected payoff $\bar r$, then the agent is said to be \emph{risk-averse} ($\beta < 0$) or \emph{risk-seeking} ($\beta > 0 $) respectively. Expected-utility (and Bayes-optimal) agents are thus special cases of risk-sensitive agents, where the certainty-equivalent is equal to expectation (i.e.\ \emph{risk-neutral} achieved with $\beta = 0$). 

A risk-averse agent (e.g. $\beta = -1$) who is sensitive to the mean and the variance would choose as follows:
\begin{enumerate}
\item[a)] Due to symmetry, the variance of either box is the same ($= 0.84$), hence the agent decides based on the mean payoffs alone, choosing the box on the right. 
\item[b)] Both boxes yield an expected payoff equal to zero, but the one on the right has a non-zero variance ($= 1$). Therefore, the agent selects the left box, which yields a certain payoff.
\item[c)] The choice of the risk-sensitive agent is undefined due to the presence of ambiguity.
\item[d)] As in case (c).
\end{enumerate}
Of special interest is the choice in case (b) where the agent penalizes variability of the reward. This type of behavior can be linked to robustness~\cite{singh2020improving} against disturbances and to portfolio theory where volatility is usually undesired. 

\begin{rem}
Note that one can consider two perspectives on the origin of risk. One one hand, one can see risk as the curvature of the utility function via Arrow-Pratt measure~\cite{pratt1964}. This requires a non-linear map from rewards to utility. On the other hand, one can directly account for the higher order moments of the plain reward when computing the certainty equivalent as we did in Equation~\eqref{eq:simple_risk_sensitive_valuation}.
\end{rem}

\paragraph{Ambiguity-sensitive.}
An ambiguity-sensitive agent spots missing information and unknown probabilities, and incorporates this lack of knowledge into the valuation of its decision.
There exist multiple ways to model ambiguity or unknown probabilities but no widely accepted formal definition~\cite{gilboa2016ambiguity}. 
For instance, in the economics literature, some ambiguity models require abandoning the additive property of probabilities by using instead capacities~\cite{schmeidler1989subjective}, others require second order beliefs~\cite{klibanoff2005smooth}(similar to hierarchical Bayesian models) and some others require multiple priors~\cite{gilboa1989maxmin}. In the rest of the paper we adopt the latter approach.

Agents facing unfamiliar situations cannot apply the standard decision-rules valid under risk due to missing information and unknown probabilities. Thus, detecting missing information is key for ambiguity-sensitive agents. 
Agents can use an ambiguity set $\Delta$, containing multiple prior beliefs about the world, to detect missing information~\cite{gilboa1989maxmin}. The size of $\Delta$ and the elements in $\Delta$ determine how much information is missing.  For example, an agent could consider three prior distributions over world-models $\Delta = \left\lbrace p_{\omega_1}(\theta), p_{\omega_2}(\theta)  p_{\omega_3}(\theta)\right\rbrace$ where
$\theta \in \Theta$. 
Each model could represent  different distributions over arbitrary colors or rewards for the unknown marbles in the situation from  Figure~\ref{fig:urn-choices} (c) and (d). Low ambiguity situations are captured when  all the models predict similar futures for a given choice, or when the ambiguity set $\Delta$ is small. However, when each model predicts a different future or there are many different models in $\Delta$ then the agent knows it is facing a highly ambiguous situation. 

Given the ambiguity set $\Delta$, an agent resolves  ambiguity in two ways. On one hand, it can adopt an arbitrary rule to select one of the interpretations. For example, an  \emph{ambiguity-averse} agent favours  pessimistic interpretations by adopting the rule of selecting the worst-case model. The valuation of such agent is
\begin{equation}\label{eq:ambiguity_valuation_simple}
Q_{\text{Amb}} (a) := \min_{p_\omega \in \Delta} \sum_\theta p_{\omega}(\theta) \sum_{o\in O} p_\theta(o | a) r(o).
\end{equation}
As denoted by the equation, there is a minimization operation selecting the model that assumes the worst expected-reward.

\begin{rem}\label{rem:bayes-optimal-valuation}
Note that the term $\sum_\theta p_\omega(\theta) \sum_{o\in O} p_\theta(o | a) r_\theta (o)$ corresponds to a Bayes-optimal agent valuation that has a prior $p_\omega(\theta)$ over models $p_\theta(o|a)$. We comment on the role of priors on a later section. 
\end{rem}

On the other hand, the agent could adopt a \emph{fall-back strategy or default policy}  independent of the particular model predictions~\cite{gilboa2009always}. For instance, in contrast to Equation~\eqref{eq:ambiguity_valuation_simple} where the final valuation directly depends on the elements inside the ambiguity set,  one could construct a strategy $\hat \pi = f( |\Delta|) $ that selects a default-policy $\hat \pi$ (according to some rule $f$) as a function that depends solely on the size of $\Delta$.  An agent freezing or stopping acting when presented with too much ambiguity could be modeled with such type of strategy.

Going back to our example, consider the choices made by a risk-neutral but ambiguity-averse agent:
\begin{enumerate}
\item[a)] This is a well-defined decision problem under risk. Since the agent is risk-neutral, it chooses the box on the right, just as an EU agent would. 
\item[b)] Like in the preceding case, this problem is unambiguous. Furthermore, since the agent is risk-neutral and the expected payoff of either box is equal to zero, it is indifferent between either choice.
\item[c)] The payoffs in the box on the right are unspecified and therefore ambiguous. The payoff of a blue marble could be equal to -1, 0, +1, or any other value, each assumption recommending a (potentially different) choice. An ambiguity-averse agent adopts a pessimistic stance and chooses the left urn with known payoffs.
\item[d)] Now the uncertainty of the box on the right is a mixture between risk and ambiguity. This could be the result of partial disclosure of information, e.g.\ where the payoffs of four of the blue marbles from case (c) were revealed to the agent. As before, the ambiguity-averse agent chooses the left box. 
\end{enumerate}

Learning reduces ambiguity to risk (see \cite{gilboa2016ambiguity} for a formal treatment). Suppose case~(c) is equal to case~(b) in disguise, and then some of the ambiguous payoffs are revealed. As shown in Figure~\ref{fig:urn-ambiguity-interpolation}, the situation in case~(d) could be regarded as an intermediate state of ambiguity between the ambiguous case~(c) and the risky case~(b). In this example, learning is materialized as the observation of unknown marbles, which updates the priors inside the ambiguity set. As more factual evidence is observed and used for learning, the priors become more similar, which renders the minimization in Equation~\eqref{eq:ambiguity_valuation_simple} more constrained and ineffective. In the infinite data limit, the priors become equal (to the best model in the model class) and the minimization has no effect, thus reducing ambiguity-sensitivity from Equation~\eqref{eq:ambiguity_valuation_simple} to plain Bayes-optimality from Equation~\eqref{eq:bayes-simple-value} (or its extension to the risk-sensitive version similar to Equation~\eqref{eq:simple_risk_sensitive_valuation}). 

\begin{rem}
As an additional example, in our experiments shown later, an ensemble of agents will play the role of the ambiguity set. This ensemble is responsible for learning from experiences and, consequently,  reducing  ambiguity to risk.
\end{rem}

\begin{figure}[t!]
  \centering
  \def\svgwidth{0.9\textwidth}
  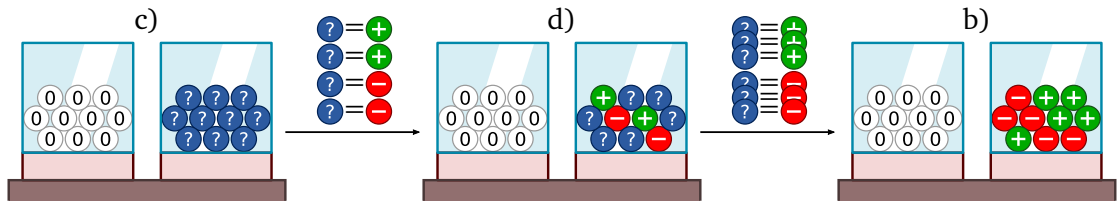
  \caption{Learning reduces ambiguity to risk. Starting from fully ambiguous contents~(c) and revealing the payoffs of the blue marbles could reduce the contents to a mixture between risk and ambiguity~(d) or even to full risk~(b).}
  \label{fig:urn-ambiguity-interpolation}
\end{figure}

\paragraph{Priors over ambiguity.} Can ambiguity be explained away through Bayesian modeling? Following Remark~\ref{rem:bayesian-ambiguity} and~\ref{rem:bayes-optimal-valuation}, one can reduce ambiguity to risk by e.g.\ placing a uniform prior over the payoffs $\{-1, 0, +1\}$ of each blue marble, \emph{and then assuming that the agent knows this\footnote{Hence, it is not the prior of an external observer.}}. Cases~(c) and~(d) then change as follows (cases a~\&~b remain the same):
\begin{enumerate}
\item[c)] The box on the right now has zero mean and variance~$=\frac{20}{3}$ (that is, $\frac{2}{3}$ per marble assuming independence). Therefore, a plain expected-utility agent would be indifferent between the options (due to risk-neutrality) and a risk-averse agent would choose the left box.
\item[d)] The box on the right has zero mean and variance~$=4$, leading to the same conclusions as for (c).
\end{enumerate}

\begin{rem}
Crucially, as highlighted in the infamous Ellsberg's experiments~\cite{ellsberg1961risk}, a single prior biased towards the unknown marbles being red is not enough to explain ambiguity aversion. To see this, one can  keep the prior fixed and switch the reward values of the green and red marbles. Since the prior is fixed, an agent predicting more red than green marbles in the ambiguous urn before the switch,  should keep this prediction after the switch. Therefore, if before the switch the agent chose the non-ambiguous urn, after the switch the agent would also swap urn preferences and pick the ambiguous urn. 
However, this contradicts  the conceptual experiments outlined in~\cite{ellsberg1961risk} where agents keep choosing the non-ambiguous option (thus being ambiguity averse) even after the switch. Multiple priors would be able to explain this type of behavior (i.e. by using Equation~\eqref{eq:ambiguity_valuation_simple}), suggesting then that a single prior is not sufficient to explain ambiguity. 
\end{rem}

All the choices are summarized in Table~\ref{tab:comp-choices}. In the table we see that the choices of the ambiguity-averse agent differ in at least one case from every other agent, even from the ones which place priors over the multiple interpretations of the blue marbles. This implies that ambiguity-sensitivity cannot be explained in terms of risk alone, even when following a Bayesian approach with subjective priors over unknown outcomes.

\begin{table}[t]
  \centering
  \begin{tabular}{lcccc}
    \toprule
    & \multicolumn{4}{c}{Case} \\
    Agent & a & b & c & d \\
    \midrule
    Expected-utility & right & indiff. & \emph{undef.} & \emph{undef.} \\
    Risk-averse & right & \emph{left} & \emph{undef.} & \emph{undef.} \\
    Bayes-optimal, with prior & right & indiff. & \emph{indiff.} & \emph{indiff.}\\
    Risk-averse, with prior & right & \emph{left} & left & left \\
    Ambiguity-averse & right & indiff. & left & left \\
    \bottomrule
  \end{tabular}
  \caption{Comparison of Choices. The table lists the preferences of the different agents for the cases shown in Figure~\ref{fig:urn-choices}. The possible choices are: left, right, indifferent, and undefined. The choices that differ from the ambiguity-averse ones are highlighted in italics.}\label{tab:comp-choices}
\end{table}

\paragraph{Summary.} Agents who are risk- and ambiguity-sensitive make qualitatively different choices (Table~\ref{tab:comp-choices}). Choices involving well-defined probabilities fall under risk, whereas choices with missing information fall under ambiguity. Risk-sensitive agents use the payoff distribution in order to arrive at a valuation; Bayes-optimal agents are special in that they only pay attention to the expected payoff. In contrast, ambiguity-sensitive agents can resort to multiple priors and  default choices in the face of conflicting interpretations. Although we provide a separate treatment for risk and ambiguity, in general an agent could be sensitive to both types of uncertainties at the same time. 

Ambiguity-sensitivity can be regarded as a shortcoming \cite{alnajjar2009ambiguity}. In this view, an agent who is incapable of placing a prior over conflicting interpretations is irrational. Alternatively, ambiguity-sensitivity can also be defended as a rational choice, because acting with confidence on an entirely made-up prior does not seem sensible \cite{gilboa2009always}. Recall the question ``are Cydophines also Abordites?''. Is this statement true with, say, 50\% probability, or do you simply not know? 

\begin{rem}
Instead of uncertain payoffs, we could have had uncertain probabilities, namely, uncertainty about the probability of drawing a particular marble with its associated reward.  This doesn't change the rationale outlined above.
\end{rem}

\begin{rem}
There are additional classifications of uncertainty relevant to machine learning:
\begin{figure}[h!]
  \centering
  \begin{tikzpicture}
    [edge from parent fork down]
    \node {Probability} [sibling distance = 2cm] 
      child {node {Objective}}
      child {node {Subjective}};
  \end{tikzpicture}
  \qquad
  \begin{tikzpicture}
    [edge from parent fork down]
    \node {Game Theory} [sibling distance = 2.5cm] 
      child {node {Player/Nature}}
      child {node {Imperfect}}
      child {node {Incomplete}};
  \end{tikzpicture}
\end{figure}
\\The first classification distinguishes between two types of probabilities; therefore it can be considered a sub-classification of risk. These can be either objective (i.e. physical/aleatoric), often interpreted as originating from a source that is external to the agent; or subjective (i.e. degrees of belief/epistemic), coming from within the agent. While the distinction has played a significant role in recent work on Bayesian deep learning~\cite{kendall2017uncertainties}, it is beyond the scope of this work. 

The second diagram depicts a selection of uncertainty distinctions made in game theory. The first and most fundamental distinction is between Nature and players: Nature is characterized by a probability distribution over strategies, whereas players have preferences but no probabilities over strategies. Then, in normal form games (say, a two-player game), one can introduce uncertainty by randomizing the payoffs (imperfect) or simply by having blank payoff entries in the game matrix (incomplete) \cite{vonNeumann1947theory,osborne1994course}. We claim that the uncertainties in game theory can ultimately be reduced to risks and ambiguities; in particular, the uncertainty of Nature's choices and in imperfect information games corresponds to risk, and the uncertainty of player choices and in incomplete information games to ambiguities. 
\end{rem}

\section{Meta-learning Risk and Ambiguity} \label{sec:meta-learning-risk-ambiguity}
This section focuses on describing the two modifications to the standard meta-training protocol that encourage risk- and ambiguity-sensitivity, respectively. We leave the experimental details to Section~\ref{sec:experimental-methodology}.

\subsection{Bayes-optimal}\label{sec:bayes-optimal}
We briefly review meta-learning in order to set the stage. For simplicity we focus on a minimal example extending the setting from the previous section with a state $S$ and a latent task parameter $\theta$. 

Consider an interaction between the task and the agent that generates the following random variables in a sequential manner:
\[
  \theta \rightarrow S \rightarrow A \rightarrow O.
\]
Here $\theta$ is a latent task parameter, $S$ and~$O$ are observations (stimulus and outcome respectively), and~$A$ is the agent's action. These are drawn from the following generative process:
\begin{equation}\label{eq:gen-process}
  \theta \sim p_\Theta (\cdot), 
  \qquad S \sim p( \cdot \mid \theta), 
  \qquad A \sim \pi(\cdot \mid s),
  \qquad O \sim T(\cdot \mid \theta, s, a).
\end{equation}
All variables except $A$ depend (causally) on the entire history. The action $A$ cannot depend on $\theta$ because it is not seen by the agent. Due to this, the observations~$S$ and~$O$ are perceived by the agent as being drawn from the marginals
\[
  S \sim p(\cdot) 
    = \sum_\theta p_\Theta(\theta) p(\cdot \mid \theta),
  \quad\text{and}\quad O \sim \bar T(\cdot \mid s, a) 
    = \sum_\theta P(\theta \mid a, s) T(\cdot  \mid \theta, a, s).
\]
This is key for meta-learning, as it encourages an agent to learn the statistical effects of the latent variable~$\theta$ implicitly. The optimal strategy (in the sense of minimal loss/regret) in light of marginalized observations is to maintain a posterior belief over the latent variable and predict/act by marginalizing over this belief---see \cite{ortega2019meta}. A meta-learner that minimizes loss will thus produce a solution that \emph{behaves indistinguishably} from predicting/acting according to the Bayesian posterior predictive distribution, which requires inferring the value of the latent variable based its observable statistical effects.

The goal of the Bayes-optimal agent is to find a policy maximizing the expected payoff. This specific decision problem requires a stimulus-dependent policy $\pi^\ast(\cdot \mid s)$, that is,
\begin{equation}\label{eq:bayes-opt}
  \pi^\ast(\cdot \mid s) = \argmax_{\pi(\cdot \mid s)} 
    \sum_a \pi(a \mid s) Q(s, a),
  \quad\text{where}\quad Q(s, a) := \sum_{o} \bar T(o \mid s, a) r(o)
\end{equation}
are the expected payoffs (Q-values) of the outcome given an initial observation and an action, and where $r$ is a reward function. Because the $\bar T(o \mid s, a)$ in the definition of $Q(s, a)$ entail computing posterior probabilities $P(\theta \mid s, a)$, the resulting agent acts \emph{as if} it were holding probabilistic beliefs over the latent task parameter~$\theta$. The agent might internally compute these beliefs, but this is, in general, not necessary~\cite{mikulik2020meta,ortega2019meta}.

To solve numerically for the objective~\eqref{eq:bayes-opt}, memory-based meta-learning optimizes a Monte-Carlo approximation w.r.t.\ the policy parameters. Specifically, consider the approximation of the expected reward $R \sim r(O)$,
\begin{equation}\label{eq:mc-approx}
  \mathbb{E}[R \mid \pi] =
  \sum_{\theta, s} p_\Theta(\theta) P(s\mid \theta) \Bigl[
    \sum_a \pi(a \mid s) Q(s, a) \Bigr]
  \approx \frac{1}{N} \sum_n r(o^{(n)}),
\end{equation}
where the $o^{(n)}$ on the r.h.s.\ are sample outcomes from the generative process~\eqref{eq:gen-process}. Maximizing this objective w.r.t.\ the policy parameters evaluated on batches generated by~\eqref{eq:gen-process} yields a Bayes-optimal policy \cite{ortega2019meta,mikulik2020meta}.

\subsection{Risk-sensitive}\label{sec:basic-risk-sensitive}
What changes to the meta-training protocol outlined above are required  so that the agent additionally cares about the higher-order moments of the return? We modify the meta-training protocol from above by tweaking the distribution over observations to be sensitive to the valuations of the agent. That is, compared to the generative process~\eqref{eq:gen-process} in the Bayes-optimal case, the distribution over the outcome~$O$ is now also conditioned on the current Q-value estimates $\hat Q$, which can be written as:
\begin{equation}\label{eq:gen-risk}
  O \sim T(\cdot \mid \theta, S, A)
  \qquad \longrightarrow \qquad
  O \sim \rho(\cdot \mid \theta, \hat Q, S, A).
\end{equation}
An interpretation of this modification is that the agent now considers itself to be embodied, as the Q-value estimates of the agent are part of the environment and these Q-values can be used to fully describe the agent. While this does not mean that the agent considers its own explicit self (e.g. source code) as part of the environment, it is considering its own behavior by including the Q-values as a component of the transition dynamics. Alternatively, the environment could contain other entities with  theory of mind~\cite{frith2005theory}, allowing for estimates of what the agent might do.

The proposition below considers a simplified case (without states and latent variables) that illustrates how to modify the environment such that the agent valuations are risk-sensitive.  First, we propose the following modification to the transition which depends on the estimate $\hat V$ (in the sequential case $\hat Q$ and $\hat V$ become the same object)  and the original transition $T(o|a)$: 
\begin{equation}\label{eq:simple_modified_transition}
 \rho(o|a,\hat V) := \frac{1}{Z} T(o|a) e^{\beta \hat V(o)}.
\end{equation}
with normalizing constant $Z$ and hyper-parameter $\beta$.
Second, since the dynamics have changed, the agent now acquires the following valuation over decisions
\begin{equation}\label{eq:value-risk-simple}
    \hat Q(a):= \sum_o \rho(o|a, \hat V) \hat V(o).
\end{equation}

\begin{proposition}\label{prop:simple-risk-sensitive-valuation}
Let $\hat V(o)$ be the current valuation that the agent assigns to $o$. Further assume that the environment is modified to be $\rho(o|a,\hat V)$ from Equation~\eqref{eq:simple_modified_transition}. Then the Q-values from Equation~\eqref{eq:value-risk-simple}  are risk-sensitive in the sense that they are a function of the expectation and the variance of $V$ under $T$ depending on $\beta$. That is
\begin{equation}
    \hat Q(a) \approx \mathbb E_{T(\cdot |a)}[ \hat V(o) ] +  \beta \mathbb{VAR}_{T(\cdot |a)} [\hat V(o)].
\end{equation}
\end{proposition}
\begin{proof}
The proof is in Appendix~\ref{appendix:risk-theory-single-step}.
\end{proof}

Compared to the Bayes-optimal case in~\eqref{eq:bayes-opt}, the valuation from  Proposition~\ref{prop:simple-risk-sensitive-valuation} is sensitive to the variance of the value estimates. We achieve sensitivity to risk by means of the objective function~\eqref{eq:value-risk-simple} and the appropriate modified transition dynamics $\rho$ from Equation~\eqref{eq:simple_modified_transition} (for a more detailed treatment see Appendix~\ref{appendix:risk_theory}).

\begin{rem}
Risk-sensitivity does not arise for any choice of $\rho$.
Other forms of $\rho$ (different from Equation~\eqref{eq:simple_modified_transition}) could exist that generate risk-sensitivity. Discovering what forms of $\rho$ generate or not risk-sensitivity is out of the scope of this work.  
\end{rem}

\subsection{Ambiguity-sensitive} \label{sec:basic-ambiguity-sensitive}
Inducing ambiguity-sensitivity requires designing a mechanism that detects and uses novelty, which we do by using an ensemble and a meta-policy.

Let $\mathcal{S}$ be a set of stimuli (i.e. the observations presented to the agent), and partition it into $\mathcal{S}_\text{risky}$ and $\mathcal{S}_\text{ambiguous}$.  Now, consider a collection of $K$ agents with policies $\pi_1, \ldots, \pi_K$ respectively, where each policy~$\pi_k$ is implemented as a high-capacity function approximator with different initialization. If each member of the ensemble optimizes the Bayes-objective~\eqref{eq:bayes-opt} where the stimuli are drawn from the distribution $P(S \mid \theta, \mathcal{S}_\text{risky})$ restricted on $\mathcal{S}_\text{risky}$, the performance of any two fully trained  policies $\pi_j$ and $\pi_k$ will be approximately equal:
\[
 \mathbb{E}[R \mid 
   \pi_j, \mathcal{S}_\text{risky}]
 \approx \mathbb{E}[R \mid 
   \pi_k, \mathcal{S}_\text{risky}],
\]
since only one Bayes-optimal solution exists.

However, if we evaluate the same policies on $\mathcal{S}_\text{ambiguous}$, then they will differ in their performance given sufficiently high capacity and sufficiently different initial parameters (see~\cite{baek2022agreement} for an empirical treatment with standard neural networks): 
\[
 \mathbb{E}[R \mid 
   \pi_j, \mathcal{S}_\text{ambiguous}]
 \not\approx \mathbb{E}[R \mid 
   \pi_k, \mathcal{S}_\text{ambiguous}].
\]
This is due to the bias-variance trade-off: models with high capacity may have low bias but typically have large variance, implying that the resulting policies differ from each other. Since the ensemble disagrees on what to predict when presented with a stimulus drawn from $\mathcal{S}_\text{ambiguous}$, said stimuli are ambiguous. Exploiting ensemble's disagreement  for detecting novel situations is common practice~\cite{osband2016deep,lakshminarayanan2017simple,osband2021epistemic}. 

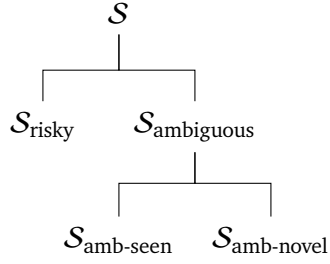
\begin{figure}[t]
  \centering
  \begin{tikzpicture}
    [edge from parent fork down]
    \node {$\mathcal{S}$} [sibling distance = 2cm] 
      child {node {$\mathcal{S}_\text{risky}$}}
      child {node {$\mathcal{S}_\text{ambiguous}$}
        child {node {$\mathcal{S}_\text{amb-seen}$}}
        child {node {$\mathcal{S}_\text{amb-novel}$}}
      };
  \end{tikzpicture}
  \caption{Partioning of stimuli. The ensemble is trained on the set of risky stimuli $\mathcal{S}_\text{risky}$. The meta-policy is trained on the set of risky and seen stimuli, i.e.\ $\mathcal{S}_\text{risky} \cup \mathcal{S}_\text{amb-seen}$. Since the meta-policy is a function of the ensemble output and not the stimulus, it cannot distinguish between stimuli in $\mathcal{S}_\text{amb-seen}$ and $\mathcal{S}_\text{amb-novel}.$}\label{fig:stimulus-partition}
\end{figure}

We leverage the ensemble's ability to detect ambiguity in order to formulate an agent that can respond to it. First, assume we have access to a subset of $\mathcal{S}_\text{ambiguous}$, the set of ambiguous stimuli, for training so that we can further partition it into $\mathcal{S}_\text{amb-seen}$ and $\mathcal{S}_\text{amb-novel}$ (see Figure~\ref{fig:stimulus-partition}).
Next, we introduce a \emph{meta-Q-function} of the form
\begin{equation}\label{eq:meta-values}
 Q^m(s,a) := f(Q^{\pi_1}(s, a), \ldots, Q^{\pi_K}(s, a)),
\end{equation}
where $Q^\pi$ are the values of policy $\pi$ and $f$ a function that can be trained with data experience. The meta-Q-function, which resembles a mixture of experts model, is used to construct  a \emph{meta-policy} (we also use the term \emph{top-policy}) $\pi^m$ of the form
\begin{equation}\label{eq:meta-policy}
  \pi^m(a \mid s) = \argmax_\pi \sum_a \pi(a|s) Q^m(s,a),
\end{equation}

The meta-policy essentially reduces the ensemble's policy profile to a single distribution over actions and it is trained by optimizing \eqref{eq:bayes-opt} with stimuli drawn from 
\[
  P(S \mid \mathcal{S}_\text{risky} \cup \mathcal{S}_\text{amb-seen})
\]
whilst holding the ensemble fixed. Training on both $\mathcal{S}_\text{risky}$ and $\mathcal{S}_\text{amb-seen}$ familiarizes the meta-policy with the distinction between risky and ambiguous stimuli (including a hypothetical set of test stimuli $\mathcal S_\text{amb-novel}$). 
It knows how to act under both an agreement and disagreement of the ensemble. If the ensemble agrees because it is presented with a risky stimulus, it will produce the optimal action distribution, which agrees with all the members of the ensemble:
\[
  \pi^m(a|s) \approx \pi_k(a|s)
  \quad \text{if $s \in \mathcal{S}_\text{risky}$ for all $k$}.
\]
However, if the ensemble is presented with an ambiguous input (including test stimuli $\mathcal S_{\text{amb-novel}}$), then the ensemble output variance will be nonzero, and thus the meta-policy will be distinct from at least one ensemble member:
\[
  \exists k \qquad \text{such that} \qquad  \pi^m(a|s) \not\approx \pi_k(a|s)
  \quad \text{if $s \in \mathcal{S}_\text{ambiguous}$}.
\]

The meta-policy's behavior under ambiguous situations ($\mathcal{S}_\text{amb-seen}$ and $\mathcal{S}_\text{amb-novel}$) highly depends on the rewards associated with the stimuli  $\mathcal S_\text{amb-seen}$. That is, the meta-policy will be cautious or audacious (that is, ambiguity-averse or -seeking), depending on whether the ensemble's response to the stimuli in $\mathcal{S}_\text{amb-seen}$ led to low or high payoffs respectively.

\begin{rem}
The meta-policy can also implement default strategies that are independent of the ensemble valuations. Doing so would require $Q^m$ to also depend directly on the state-action pair $(s,a)$. That is
\begin{equation}\label{eq:meta-values-extra-state}
 Q^m(s,a) := f(s, a, Q^{\pi_1}(s, a), \ldots, Q^{\pi_K}(s, a)),
\end{equation}
This type of valuation has the capability of implementing a context-dependent ambiguity-dependent meta-policy. However, practitioners must be careful when training such meta-policy since it opens the possibility for  bypassing all the information coming from the ensemble. This bypass would render the meta-policy useless when presented with  a truly novel ambiguous situation $\mathcal S_{\text{amb-novel}}$. How to solve this problem for a more robust context-dependent ambiguity-sensitivity is left for future work.
\end{rem}

\paragraph{Summary.} We have shown two simple mechanisms to induce risk and ambiguity sensitivity. Both of them depend on the history of the agent's observations, actions and rewards. This way agents can become uncertainty-seeking or -averse depending on the experiences within a particular context. This could be a promising line of future work since it would allow for flexible data-dependent risk- and ambiguity-sensitivity.

\section{Experimental Methodology}\label{sec:experimental-methodology}
Here, we describe implementation details of the agents outlined above. The crucial difference  between the previous section and this one is two-fold. First, we consider the sequential decision-making case and, second,  we also  explain the experimental details  and methodology that we use for our experiments from Section~\ref{sec:experimental-results}.

\subsection{Agent and Training Details}

\subsubsection{Vanilla Meta-Learning: Baseline Methodology}
We consider Markov Decision Processes (MDPs) which are defined as the tuple ($\mathcal X, \mathcal A, T_\theta, r_\theta, \gamma)$ where $\mathcal X$ is the state space, $\mathcal A$ is the action space and $\gamma$ the discount factor. The transition function $T_\theta: \mathcal X \times \mathcal A \rightarrow P(\mathcal X)$ and reward function $r_\theta: \mathcal X \times \mathcal A \rightarrow \mathbb R$  depend on the latent parameters $\theta$ that completely specify the task.  The agent's policy depends on the current state $x_t$ and a history variable $h_t$ of  state-action-reward time-step data. Our memory-based architectures summarize history via $h_t := (m_t, a_{t-1}, r_{t-1})$, i.e. a memory variable $m_t \in \mathcal M$ and the  action-rewards from the previous time-step $(a_{t-1}, r_{t-1})$~\cite{wang2016learning}. In our case, the memory space $\mathcal M$ is $\mathbb R ^ d$, and corresponds to the memory of an LSTM trained using Backpropagation Through Time. Our policies are of the form $\pi_\omega(a_t, m_{t+1} |x _t, h_t)$. Trajectories $\tau := (x_0, m_0, a_0, \dots x_H, m_H)$ are distributed according to 
\begin{equation}\label{eq:traj_prob}
  p\left(\tau \big| \theta, \pi_\omega, \hat h \right):=  p_\theta(x_0) \prod_{t=0}^{H-1} T_\theta(x_{t+1} | x_t, a_t)\pi_\omega(a_t, m_{t+1} | x_t, h_t), 
\end{equation}
given  the latent parameters $\theta$, the agent's policy $\pi_\omega$ with parameters $\omega$ and some arbitrary and \emph{fixed} initialization of the agent's initial history $\hat h:=(m_0$, $a_{-1}$, $r_{-1})$. 

Given a particular task $\theta$ and policy $\pi_\omega$,  the $Q$-function is defined as
\begin{align}\label{eq:Qvalues_specific_theta}
Q^{\pi_\omega}_\theta(x, a, h) & :=  \mathbb{E}  \left[  \lim_{H \rightarrow \infty} \sum_{t=0}^H \gamma^t r_\theta (x_t, a_t) \, \bigg| \, x = x_0, a = a_0, h=\hat h\right].
\end{align}
The value function is computed as  $V^{\pi_\omega}_\theta (x, h) := \max_a Q^{\pi_\omega}_\theta(x, a, h)$.  The optimal policy for a particular $\theta$ and $x$ is $\pi^*_\theta = \arg \max_\pi V^\pi_\theta (x, h)$ which is the same for all $x \in \mathcal X$~\cite{puterman2014markov}.

We use the R2D2 learning algorithm~\cite{kapturowski2018recurrent} to meta-train our memory-based agents. In short, R2D2 learns a Q-function via $n$-step temporal-difference error updates with additional modifications to the loss function to ease the optimization procedure (see~\cite{kapturowski2018recurrent} for more details). 
We use the standard meta-training protocol, namely,  we sample an environment $\theta \sim P_\theta(\cdot)$ and collect an episodic trajectory $\tau$ used to approximate Equation~\eqref{eq:Qvalues_specific_theta} with the return $\hat Q_\theta(\tau) := \sum_{t=0}^H\gamma^t r_\theta(x_t, a_t)$. The approximation $\hat Q_\theta$ is used as a signal for learning our agent's Q-vector with each entry being the q-values for a particular action 
\begin{equation}\label{eq:Qfunction_parametric}
  \mathbf Q_\omega (x_t, h_t) := \left[ Q_\omega (x_t, a^1, h_t),  \dots Q_\omega (x_t, a^{|\mathcal A |}, h_t)  \right].
\end{equation}
For more details about the loss and optimization procedure see~\cite{kapturowski2018recurrent}. 
    Our agent implementation instantiates the policy via $\pi_\omega(a' | x_t, h_t) = \delta_{a'a^\star}$ where $ a^\star:= \argmax_a Q_\omega (x_t, a, h_t) $ and  $\delta$ is the Kronecker delta.
The \emph{memory} $m_t$ is  reset at the beginning of each episode and it is updated along the trajectory. 

\subsubsection{Methodology for Meta-Learning Risk-Sensitivity}\label{sec:summary_methodology_risk}

\begin{figure}[t]
\begin{minipage}{0.49\textwidth}
\begin{algorithm}[H]
\footnotesize
\caption{An off-policy algorithm to learn risk-sensitive policies.}\label{alg:risk}
\begin{algorithmic}[1]
\State{\textbf{Asynchronous data collection procedure:}}
\For{$s \gets s+1$ until $s==max \: steps$}
    \State Sample $\theta \sim P_{\Theta_\text{risky}}$ and the first state $x_0$
    \State Initialize the agent's memory $m_0 \sim \mathcal M$
    \For{$t \gets t + 1$ until episode ends}
        \State Sample $a_t, m_{t+1} \sim \pi_\omega(\cdot | x_t, h_t)$
        \State Update $r_t \gets r_{\theta}(x_t, a_t)$
        \State Sample $x^{k}_{t+1} \sim T_{\theta} (\cdot | x_t, a_t)$ for $k \in \{0, ..., N\}$
        \State Sample $\tilde{k} \sim softmax_{k \in \{0, ..., N\}} \: V(x^{k}_{t+1}, h_{t+1})$
        \State Update $x_{t+1} \gets x^{\tilde{k}}_{t+1}$
        \State Add $(x_t, r_t, a_t, x_{t+1})$ to the replay buffer
    \EndFor
\EndFor
\State{\textbf{Asynchronous Training:}}
\State Sample batch from replay buffer
\State Update parameters with R2D2 (see \cite{kapturowski2018recurrent})
\end{algorithmic}
\end{algorithm}
\end{minipage}
\hfill
\begin{minipage}{0.49\textwidth}
\begin{algorithm}[H]
\footnotesize
\caption{An off-policy algorithm to learn ambiguity-sensitive policies.}\label{alg:ambiguity}
\begin{algorithmic}[1]
\State{\textbf{Asynchronous data collection procedure:}}
\State \textbf{Input}: Ensemble of $K$ independently trained Q-functions $(Q_{\omega_k})_{k \in \{0, ..., K\}}$ on $P_{\Theta_\text{risky}}$.
\For{$s \gets s+1$ until $s==max \: steps$}
    \State Sample $\theta \sim P_{\Theta_\text{amb-seen}}$ and the first state $x_0$
    \State Initialize top-level agent's memory $m_0 \sim \mathcal M$
    \State Initialize ensemble agents' memories $m^{k}_0\sim \mathcal M$
    \For{$t \gets t + 1$ until episode ends}
        \State Infer $Q_{\omega_k}(x_t, m^k_t, a_{t-1}, r_{t-1})$ for $k \in \{0, ..., K\}$
        \State Update $m^k_{t+1} \sim \pi_{\omega_k}(\cdot | x_t, m_t, a_{t-1}, r_{t-1})$
        \State Sample $a_t, m_{t+1} \sim \pi_{meta}(a_t, m_{t+1} | f(Q^{ens}), m_t, a_{t-1}, r_{t-1})$
        \State Update $x_{t+1} \sim T_{\theta} (x_{t+1} | x_t, a_t)$
        \State Add $(x_t, r_t, a_t, x_{t+1})$ to the replay buffer
    \EndFor
\EndFor
\State{\textbf{Asynchronous Training:}}
\State Sample batch from replay buffer
\State Update top-policy parameters with R2D2 (see \cite{kapturowski2018recurrent})
\end{algorithmic}
\end{algorithm}
\end{minipage}
\end{figure}

The specific mechanism that we use to instantiate Equation~\eqref{eq:gen-risk} is to modify the transition dynamics depending on the current Q-value estimations of the agent. The basic idea is to construct $\rho_{\theta\omega}$ that deviates from the natural transition dynamics $T_\theta$ in a way that  favors (risk-seeking) or is against the agent's expectations (risk-averse)~\cite{mohammedalamen2021learning}. For example, for the risk-seeking case,  the probability of the next state $x'$ is increased (on average), i.e., $\rho_{\theta, \omega} (x' |x,a) > T_\theta(x' | x, a)$, when the value of $x'$ is \emph{higher} than the average value i.e.,  $V_\omega (x', \tilde h) > \mathbb E_{x'} [V_\omega (x', \tilde h)]$ for the current memory-action-reward (history) context $\tilde h$ (which is fixed) where $V_\omega(x', \tilde h):= \max_a [ \mathbf  Q_\omega (x', \tilde h)]_a $.
Below we propose two mechanisms satisfying this condition. 

\paragraph{Mechanism 1.}
A naive approach could be (for the risk-seeking case) to draw $N$ samples from $T_\theta$ i.e., $\mathcal D_x = \{x'_i \}_{i=1}^N$ with each $x'_i \sim T_\theta$, and select the one that has highest value according to the agent's value function $x_{\text{selected}}' = \argmax_{x' \in \mathcal D_x} V_\omega(x', \tilde h)$. Note that with this modification $x'_{\text{selected}}$ are sampled from a distribution $\rho_{\theta,\omega}$ (which we do not define) different from $T_\theta$. Larger $N$ creates environments that can deviate more from $T_\theta$. Smaller $N$ brings $\rho_{\theta\omega}$ closer to $T_\theta$. See extreme value theory for a theoretical treatment~\cite{de2006extreme}.

\begin{rem}\label{rem:mechanism1}
For $N=1$ we have, by construction, that $\rho_{\theta\omega} = T_\theta$.  In our initial pilot experiments, we found $N=2$ to generate too extreme risk-sensitive behavior in preliminary experiments, therefore, we propose an alternative smoother mechanism below.
\end{rem}

\paragraph{Mechanism 2.}
The problem described in Remark~\ref{rem:mechanism1} can be solved by using a smoother parameterization using $\beta$ (instead of $N$) in the following 
\begin{equation}\label{eq:rho}
    \rho_{\theta\omega} (x_{t+1} | x_t, a_t, h_t) := \frac{T_\theta(x_{t+1}| x_t, a_t) e^{\beta V_\omega (x_{t+1},  h_t)}}{Z}
\end{equation}
where $Z$ is just normalizing over $x_{t+1}$. Positive $\beta$ generates streams of experiences that are above the agent's current expected values (on average), whereas for negative $\beta$ we obtain the opposite. The formula is a softmax(min) approximation of the max(min) of $V_\omega(x', \tilde h)$, respectively. Setting $\beta = 0$ recovers $\rho_{\theta,\omega} = T_\theta$, $\beta = +\infty$ gives $x_{\text{selected}}' = \argmax_{x' \in \mathcal D_x} V_\omega(x', \tilde h)$ and $\beta = -\infty$ gives $x_{\text{selected}}' = \argmin_{x' \in \mathcal D_x} V_\omega(x', \tilde h)$, as expected. In practice, as we don't have access to the whole distribution $T_\theta$, we draw N samples (as in mechanism 1) and apply the softmax to the values of the sampled next states. We show this approximation converges to Equation~\eqref{eq:rho} when $N \to +\infty$ (see Proposition~\ref{prop:convergence_to_rho} in Appendix~\ref{appendix:protocol_details_risk}). The concrete procedure is described in Algorithm~\ref{alg:risk}.

Figure~\ref{fig:risk_one_step}C depicts the sampling methodology that we employ in practice. This is an approximation to Equation~\eqref{eq:rho}.  As shown, several proposal samples are drawn from the risky urn according to $T$. This corresponds to several next-states $s'$ which are evaluated according the the agent's value function. With this information a proxy distribution is constructed which depends on the $\beta$ parameter. For positive $\beta$ (yellow bars) we see that the proxy distribution is skewed towards the green marbles, whereas for negative $\beta$ (red bars) it is skewed towards the red marble. Finally, using the proxy distribution we sample the final sample which, in this example, corresponds to a green marble for positive $\beta$, or to a red marble for negative $\beta$.

\begin{rem}
The transition dynamics from Equation~\eqref{eq:rho} link to the concept of entropic risk measures. See Appendix~\ref{appendix:risk_theory} for a more formal treatment.
\end{rem}

\subsubsection{Methodology for Meta-Learning Ambiguity-Sensitivity} \label{sec:methodology_ambiguity}
To instantiate Equation~\eqref{eq:meta-policy} we make use of an ensemble of agents and a meta-policy, all of them being independent R2D2 agents.

\paragraph{Task partitions.} In Section~\ref{sec:basic-ambiguity-sensitive} we described a partition of the stimuli $\mathcal S_\text{risky}, \mathcal S_\text{amb-seen}, \mathcal S_\text{amb-novel}$. In our experiments, stimuli $s_t$ at time $t$ are tuples $s_t := (x_t, m_t, a_{t-1}, r_{t-1})$. These are obtained from trajectories $\tau$ distributed according to Equation~\eqref{eq:traj_prob} where the latent parameters $\theta$ are sampled from one of the following three task distributions $P_{\Theta_{\text{risky}}}, P_{\Theta_{\text{amb-seen}}}, P_{\Theta_{\text{amb-novel}}}$. Task distributions  have disjoint sample spaces  $\Theta_{\text{risky}}, \Theta_{\text{amb-seen}}$ and $\Theta_{\text{amb-novel}}$. This way, we can properly teach the agent how to act in known and unknown situations.

\paragraph{Ensemble.} We use an ensemble of $K$ Q-functions denoted as $\mathbf Q^{\text{ens}} := [\mathbf Q_{\omega_1}, \dots \mathbf Q_{\omega_K} ]$ where each $\mathbf Q_{\omega_i}$  follows Equation~\eqref{eq:Qfunction_parametric}. We independently train each member of the ensemble with R2D2 on the set $\mathcal S_{\text{risky}}$ which is composed by sub-trajectories extracted from $\tau \sim p(\cdot | \theta_{\text{risky}}, \pi_{\omega_i}, m_0, a_{-1}, r_{-1})$ following Equation~\eqref{eq:traj_prob}. Therefore, each member of the ensemble learns to be Bayes-optimal for the stimuli $\mathcal S_{\text{risky}}$ with a prior over tasks $P_{\Theta_{\text{risky}}}$.

\paragraph{Top-policy.} The meta-policy is of the form $\pi_\text{meta} (a_t, m_{t+1} | g( \mathbf Q^{\text{ens}}), m_t, r_{t-1}, a_{t-1} )$ where $g: \mathbb R^{| \mathcal A | \times K} \rightarrow \mathbb R^{| \mathcal A | \times \ell}$. Naively, the function $g$ can simply be the identity function in which case $\ell = K$. This might not be the best choice since it carries the risk of leaking state information (i.e. the ensemble Q-values give information about the state) which is an undesirable property because it enhances the meta-policy's sensitivity to the state. Since our aim  is that the meta-policy is invariant to the state but sensitive to the \emph{state-novelty}, we should minimize for state-information leakage. We do so by further compressing the ensemble's output by letting $g$ only output the first and second order moments of the Q-values for each action. In this case $\ell = 2$.  In addition, in order for the meta-policy to be always in-distribution it needs to know how to act under both types of stimuli $\mathcal S_\text{risky}$ and  $\mathcal S_{\text{amb-seen}}$, by using the ensemble---which has already been trained on  $\mathcal S_\text{risky}$---and by learning any new behavior or default policy when encountering novel stimuli in $\mathcal S_{\text{amb-seen}}$. We describe the exact procedure in Algorithm~\ref{alg:ambiguity}.

The neural network architecture that we use for our Q-learners (even the meta-policy) is a recurrent neural network with a MLP/CNN (depending on the nature of the input) torso, an LSTM layer~\cite{hochreiter1997long} and an MLP head. We use gradient normalization to avoid training instabilities. 

\subsection{Experiments and Environment Details}
Below we describe four experiments on decision-making scenarios involving urns (see Figure~\ref{fig:urn-choices}) and one experiment involving a grid-world environment.

\subsubsection{Urn Experiments}

In all of our urn experiments the agent needs to choose between two urns, i.e. $| \mathcal A | = 2$, each urn containing $M = 10$ colored marbles (see Figure~\ref{fig:risk_one_step}A for an example). After the agent's choice, a marble is randomly drawn from the chosen urn and a color-dependent reward is given.  White marbles give no reward ($r=0$), whereas red and green marbles give negative ($r=-1$) and positive ($r=+1$) rewards, respectively. Let  the tuple $c_i = (M_w, M_g, M_r)$ be the \emph{urn configuration} of urn $i$ denoting the number of white ($M_w$), blue ($M_g$) and red ($M_r$) marbles. Clearly  $\sum_i M_i = M$. We use this notation also in Table~\ref{tab:experiments} where we summarize our experiments involving urns.

\begin{table}[t]
\scriptsize
\begin{tabular}{p{1.4cm}p{1.7cm}p{5.8cm}p{5.8cm}} 
 \textbf{Uncertainty} &
 \textbf{Probabilities} &
 \textbf{Left urn configurations} (certain) &
 \textbf{Right urn configurations} (stochastic) \\ \toprule
Risk \newline (Alg. 1)  &
 Described \newline (one-step) &
   \textbf{Training:} all $M$ marbles are either \white{W}, \green{G} or \red{R}  e.g., (\white{0}, \green{10}, \red{0}). \newline
   \textbf{Testing:} Always white marbles i.e., (\white{10}, \green{0}, \red{0}) &
  \textbf{Training:} uniformly random samples \white{W}, \red{R} and \green{G}, e.g., (\white{5}, \green{1}, \red{4}) \newline \textbf{Testing:} Same as training. \\
  \midrule
Risk \newline (Alg. 1) & 
  Experiential (sequential) &
  Same as above and urn configuration fixed during an episode. &
  Same as above and urn configuration fixed during an episode. \\
  \midrule
Ambiguity \newline (Alg. 2) & 
  Described (one-step)  &
   \textbf{Ensemble Training:} all $M$ marbles are either \white{W}, \green{G} or \red{R}, never \blue{B} or \yellow{Y} e.g., (\white{0}, \green{0}, \red{10}, \blue{0}, \yellow{0}) \newline
   \textbf{Meta-Policy Training:} all $M$ marbles are either \white{W}, \green{G}, \red{R} or \blue{B}, never \yellow{Y} e.g., (\white{0}, \green{0}, \red{10}, \blue{0}, \yellow{0}). \newline
   \textbf{Full Architecture Testing:} Always white marbles i.e., (\white{10}, \green{0}, \red{0}, \blue{0}) &
   \textbf{Ensemble Training:} uniformly random samples \white{W}, \green{G} or \red{R}, never \blue{B} or \yellow{Y} e.g., (\white{3}, \green{4}, \red{3}, \blue{0}, \yellow{0}) \newline
   \textbf{Meta-Policy Training:} uniformly random samples \white{W}, \green{G}, \red{R} or \blue{B}, never \yellow{Y} e.g., (\white{2}, \green{2}, \red{2}, \blue{4}, \yellow{0}). \newline
   \textbf{Full Architecture Testing:} uniformly random samples excluding \blue{B} and \white{W} i.e., (\white{0}, \green{4}, \red{2}, \blue{0}, \yellow{4})  \\ \midrule
Ambiguity \newline (Alg. 2) &
    Experiential  (sequential) &
    Same as above and urn configuration fixed during an episode. &     Same as above and urn configuration fixed during an episode.\\ \bottomrule 
\end{tabular}
\caption{\textbf{Urn configurations used for training and testing in the risk and ambiguity-sensitivity experiments.} The first column denotes the type of uncertainty and the algorithm we used for training. The second column denotes how the probabilities are communicated to the agent i.e., fully description (one-step decision-making problem) or experienced by exploration (sequential decision-making problem). Third and fourth columns describe the urn configurations with the following labeling: white \white{W}, green \green{G}, red \red{R}, yellow \yellow{Y} and blue \blue{B} marbles. Note that the left urn is certain in the sense that all marbles are always of the same color, noting that the color could be in itself unknown. } \label{tab:experiments}
\end{table}
\normalsize

\paragraph{Experiments involving risk.} The aim of our experiments on risk is to explore how an agent trained with Algorithm~\ref{alg:risk} is  sensitive to reward variability. For this reason, during training we expose the agent to situations where the urn on the left, tagged \emph{certain}, carries deterministic rewards (all marbles are the same) once the content is known. In contrast, the urn on the right, tagged \emph{risky}, gives stochastic rewards since it contains marbles with different colors. During testing, we fix the urn on the left to be always full of white marbles, serving as a good neutral baseline, and the urn on the right still being random. Agents trained under Algorithm~\ref{alg:risk} should be sensitive to the variability of the stochastic urn. See first row in Table~\ref{tab:experiments}.

\paragraph{Experiments involving ambiguity.} The aim  of our experiments on ambiguity is to show how the resulting agent reacts to unknown situations when trained using Algorithm~\ref{alg:ambiguity}. Different stimuli types arise by using  different marble colors: stimuli in $\mathcal S_\text{risky}$ arise from urns with red, white and green marbles. Stimuli in $\mathcal S_\text{amb-seen}$ arise by adding blue marbles to the mix whose reward values can be varied to induce different types of ambiguity sensitivity. Finally, stimuli in $\mathcal S_\text{amb-novel}$ are used in our testing scenario, where we remove the blue marbles and add yellow marbles. The latter have never been seen by the agent during training, not even the by meta-policy. See the third row in Table~\ref{tab:experiments} and the top row in Figure~\ref{fig:ambiguity_one_step}.

\paragraph{Decisions from description and from experience.} Researchers in behavioral economics commonly give  to human subjects all the necessary information to compute the probabilities of events\cite{kudryavtsev2012description,hertwig2004decisions}. Hence, these types of decision-making scenario are commonly referred to as \emph{decisions from description} and do not require exploration. This is in contrast to \emph{decisions from experience} in which full information is not given and exploration is required. This distinction is useful to highlight the effects of risk and ambiguity in the context of exploration or no exploration. 

\paragraph{Experiments with described probabilities} In a subset of our experiments tagged as `described probabilities', we also give all the necessary information to our meta-trained agents.  We do so by including in the agent's observation  a matrix $A_i \in \mathbb R^{M \times j}$  describing the contents of the $i$-urn, i.e. the color of each marble, where there are $M$ rows (one per marble) and each row is a $j$-dimensional one-hot vector that encodes the color. This equips the agent with  all necessary information to choose the optimal action in the first time-step. Subsequent trials would only force the agent to make the same decision again, hence, it is enough that our experiments on described probabilities only expose the agent to a single decision step. 

\begin{figure}[t]
    \centering
    \includegraphics[width=\textwidth]{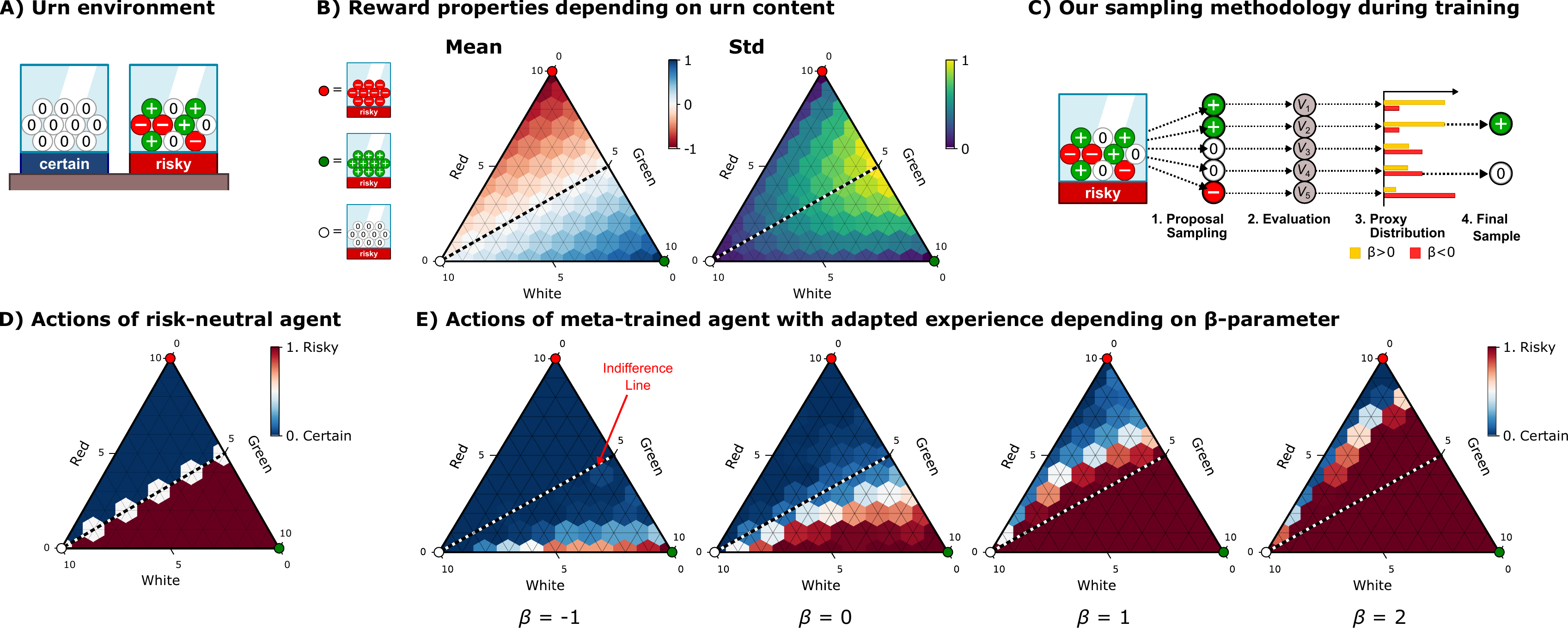}
    \caption{\textbf{Experiment 1: Risk-sensitivity with described probabilities.} \textbf{A)} Environment illustration, see Table~\ref{tab:experiments} for more details on the urn configurations. \textbf{B)} Mean and standard deviation of the reward of all possible urn configurations (see Section~\ref{sec:experimental-methodology} and~\ref{sec:experimental-results} for more details). \textbf{C)} Illustration of our approximation to Equation~\eqref{eq:rho}, see Section~\ref{sec:summary_methodology_risk} for more details. \textbf{D)} Choice behavior of an ideal risk-neutral agent when choosing between an baseline  urn (the certain urn in blue) with $0$ expected reward and an urn with configuration depicted in the triangle (tagged as risky in red). When both have the same expected value the agent is indifferent (white color), whereas it chooses the urn with highest expected reward otherwise. \textbf{E)} Choice behavior of our trained agents with different $\beta$ conditions using Algorithm~\ref{alg:risk} and environments in Table~\ref{tab:experiments}. As we see, $\beta$ controls the agent's risk-sensitivity. For example, for $\beta=-1$ the agent chooses most of the time the certain urn as denoted by the blue area. Analogous reasoning follows for the other $\beta$-conditions.}
    \label{fig:risk_one_step}
\end{figure}

\paragraph{Experiments with experiential probabilities.} Similar to multi-armed bandits, we consider the case where the agent needs to explore to discover the contents of the urns. Thus, full urn contents are not shown as in the described probabilities scenario, but instead we only show  the information of the sampled marble (a one-hot vector denoting the color) from the chosen urn at each time step.  Since exploration is required, we expose the agent to multiple time steps $H=20$. 

In summary (see Table~\ref{tab:experiments}), we conducted the following 4 experiments:
\begin{itemize}
    \item \textbf{Experiment 1}: Risk sensitivity  with described probabilities.
    \item \textbf{Experiment 2}:  Risk sensitivity  with experiential probabilities.
    \item \textbf{Experiment 3}: Ambiguity sensitivity with described probabilities.
    \item \textbf{Experiment 4}: Ambiguity sensitivity with experiential probabilities.
\end{itemize}

\subsubsection{Gridworld Experiments}
The aim of the grid-world experiments is to provide further evidence that our proposed methodology on ambiguity-sensitivity scales to more complex environments with bigger state spaces.
The left panel of Figure~\ref{fig:ambiguity_grid_world} depicts an example grid-world. The goal of the agent is to pick up the marbles that give positive reward (green) while avoiding negative reward (red). At the start of the episode there are $M=20$ marbles with randomly selected colors and locations in the room (all equally likely).
This is a similar setting to the urn experiments, but here the agent also needs to move around with four actions $|\mathcal A| = 4$ (up, down, left, right).

The agent has an egocentric view and can see up to two tiles in all directions. The resulting observation includes 25 tiles in total and is presented as a tensor of shape $(l, 5, 5)$, where the leading dimension is a one-hot encoding of the $l$ different tile types and the second and third dimensions are the $x$- and $y$-directions. For more information on the details of the environment see~\cite{pycolab}.

\paragraph{Training conditions.} The full agent is trained according to Algorithm~\ref{alg:ambiguity}. In particular, the ensemble is trained using only green ($+1$ reward) and red ($-1$ reward) marbles in analogy with the urn experiments (see third row in Table~\ref{tab:experiments}). The top-level agent can observe the raw ensemble $Q$-values (meaning $f$ is the identity function, see Section~\ref{sec:methodology_ambiguity}) and it is trained on blue marbles as well, which we set to have either positive ($+1$) or negative ($-1$) reward. As we show in the Results section, this induces ambiguity-seeking or ambiguity-averse behavior respectively. 

\paragraph{Testing conditions.}
At test time, we introduce a new marble (gray), which neither the ensemble nor the top-level agent saw during training.

\section{Experimental Results}\label{sec:experimental-results}

Before going into the results, we explain how to read the triangle plots depicting the agent's decision-making choices a wide range of urn configurations. 

\paragraph{Reading the triangle plots.} We use Figure~\ref{fig:risk_one_step}B as an example. In this figure  we show the reward's mean and standard deviation for all possible urn configurations of the risky urn. Focusing on the first triangle, its corners denote deterministic urns with $M=10$ red (top corner), white (bottom-left) and green (bottom-right) marbles. Urns with mixed contents lie in the interior and on the sides of the triangle. The axes' direction can be determined by the small indent near the number $5$. For example, the horizontal indent indicates that all possible urns with the same number of red marbles lie horizontally. The red-blue color code denotes the value of the mean reward for that urn configuration. 
We depict in a dashed line the urn configurations with zero expected reward having equal number of red and green marbles. The second triangle shows the reward's standard deviation. Its highest value (in bright yellow) corresponds to the urn with $5$ red and $5$ green marbles.  

\paragraph{Evaluation procedure in urn experiments.} In all of our four urn experiments we train 5 agents per experiment with different random seeds---according to Algorithms~\ref{alg:risk} and \ref{alg:ambiguity} and Table~\ref{tab:experiments}. Next, we test each of them \emph{on all possible combinations of marble colors}  allowed by the test conditions outlined in the table. Note that the agent may  receive as an input the tensor $A_i$ containing the information about the urn contents (see Experiments with described probabilities). Then, by construction there exist multiple permutations (row orderings) that describe the same urn configuration. These permutations can generate small variations in choice behavior (since our neural networks are sensitive to the ordering), which we diminish by sampling \emph{for each combination} $G=100$ permutations (row orderings) and taking the average. 
So, in summary, each dot in a triangle plot is the average of $5$ agents and $100$ permutations per agent. Since there are $66$ urn configurations, our agents are exposed to $33000$ decision-making situations per triangle plot.

\begin{figure*}[t!]
    \centering
    \includegraphics[width=\textwidth]{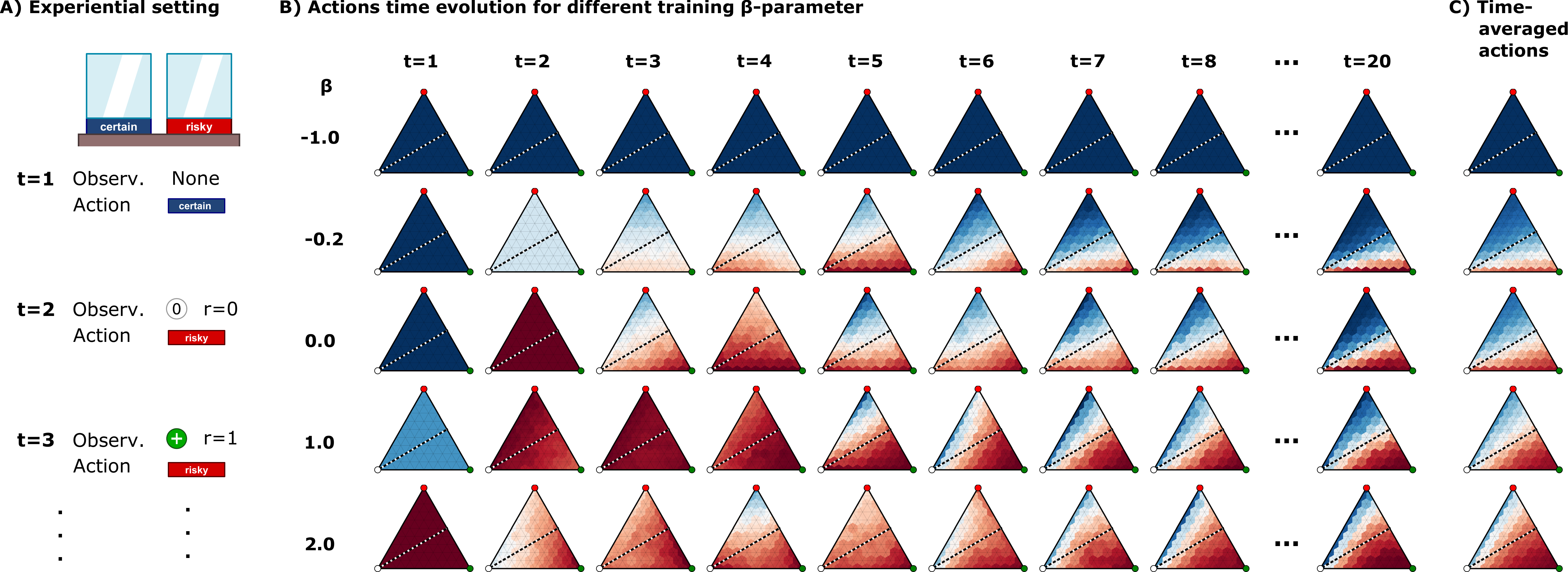}
    \caption{\textbf{Experiment 2: Risk-sensitivity with experiential probabilities.} \textbf{A)} Illustration of the sequential setting. At $t=1$ the agent receives an empty observation denoting zero knowledge about the urn contents. As time advances and choices are made the agent observes the sampled marbles from the selected urn. \textbf{B)} Choice behavior for the agents trained with Algorithm~\ref{alg:risk} on conditions of Table~\ref{tab:experiments}. Rows denote different agents trained with different $\beta$-conditions and columns denote the choice behavior at each time-step. As can be seen higher $\beta$ generates risk-seeking behavior that chooses more often the risky urn (in red) whereas negative $\beta$ generates risk-averse behavior by choosing more often the certain urn (in blue). \textbf{C)} Time-averaged behavior.}
    \label{fig:risk_multi_step}
\end{figure*}

\paragraph{Evaluation procedure in grid-world experiments.} To evaluate ambiguity-seeking and ambiguity averse agents in the grid world environments, we introduce a novel marble (gray) and monitor how many gray marbles the agent picks up in an 80-step episode. Since there are $64$ accessible tiles in the grid world, $80$ steps are enough for the agent to pick up all the marbles it wants to pick up for most initial marble configurations. If the agent picks up most of the gray marbles, this indicates ambiguity-seeking behavior, while avoiding them suggests ambiguity-averse behavior. As baselines for comparison, we track the number of known  marbles with positive reward (green) and negative reward (red) picked up by the agent. To make the results robust, we repeat the analysis for $200$ different random initial marble configurations.

\subsection{Results Experiment 1: Risk sensitivity with described probabilities}

\paragraph{Risk-neutral baseline.} Figure~\ref{fig:risk_one_step}D shows the choice behavior of an expected utility maximizer in the testing environment.   The choice behavior for each urn composition is specified with a single scalar from $0$ (blue)---i.e. $100\%$ chance of choosing the certain urn with white marbles---to $1$ (red)---i.e. $100\%$ chance of choosing the risky urn with composition specified in the triangle plot. 
The agent chooses the risky one with probability $1$, when the risky urn contains  a greater number of green marbles than red marbles. The agent is indifferent between both urns,     when both urns have the same expected utility (shaded line). This behavior is in line with an expected-utility maximizer and is the baseline that we use to compare with the other risk-sensitive agents.

\paragraph{Risk-sensitive results.} Figure~\ref{fig:risk_one_step}E shows the average choice behavior of agents trained with the methodology above for different $\beta$ values. As can be seen $\beta=-1$ generates a risk-averse agent that chooses  the certain urn most of the time. As $\beta$ increases the average behavior becomes more and more risk-seeking. This more acute for urn configurations with high reward variability. 

\begin{rem}
Note how, for $\beta=0$, we obtain an agent exhibiting close to risk-neutral behavior but not fully. We have the hypothesis that this effect is due to training with stochastic gradients or due to the overestimation problem in RL~\cite{leibfried2017information}. The study of this effect is out of the scope of this work. 
\end{rem}

\subsection{Results Experiment 2: Risk sensitivity with experiential probabilities}

Figure~\ref{fig:risk_multi_step}B shows the average choice behavior in different $\beta$ conditions for different time-steps. As can be seen, negative $\beta$ values induce risk-averse behavior and, as we increase $\beta$, the agent becomes more risk-seeking. Looking at the time direction, we can see the uniform color in all configurations of early time-steps, suggesting that the agent is fairly insensitive to the initial observations. For example, for $\beta \leq 0$ the agent always chooses the certain urn at $t=0$. This is a smart strategy because the left urn is always certain, thus knowing its value early on is useful for deciding whether to stick to that urn (as in $\beta = -1.0$). After $t=8$ the average choice behavior is fairly stable until the end of the sequence at $t=20$.  Figure~\ref{fig:risk_multi_step}C shows the time-averaged choice behavior obtained from averaging all decisions from all time steps. This is just a compressed version of all time-step data. Here we can also observe how risk-sensitivity evolves depending on the $\beta$ values.

\begin{figure*}[t]
\centering
\includegraphics[width=\textwidth]{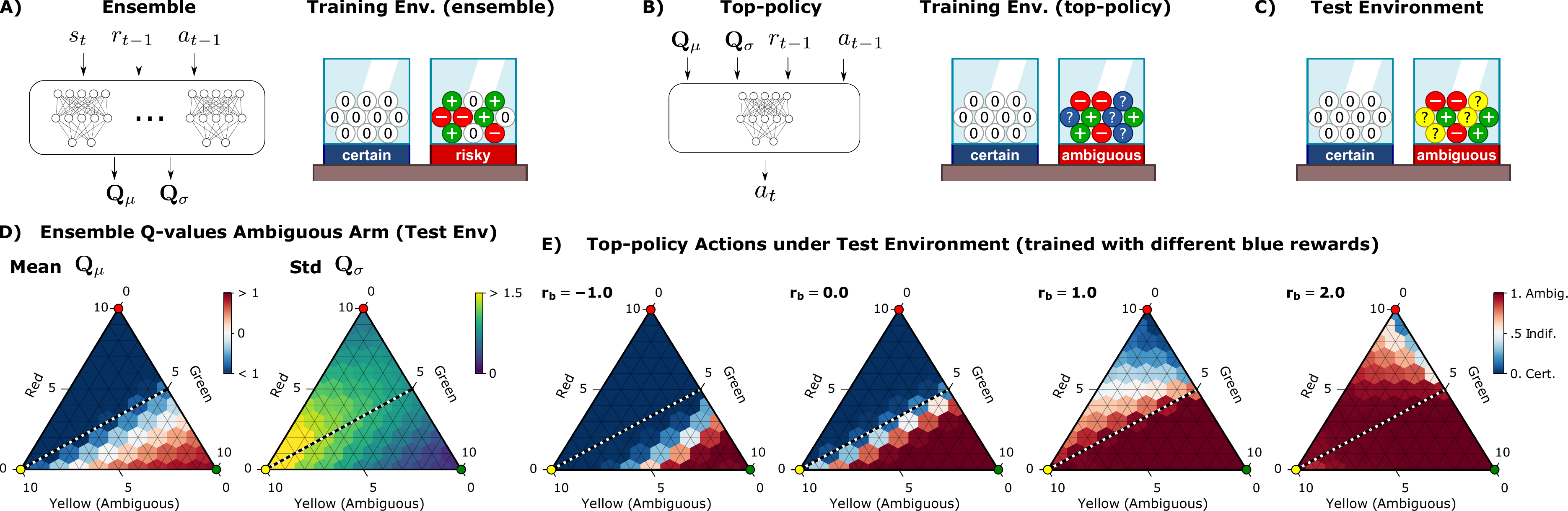}
\caption{\textbf{Experiment 3: Ambiguity-sensitivity with described probabilities.} \textbf{A)} Illustration of the ensemble and the training environment used to train it.  \textbf{B)} Illustration of the top-policy (or meta-policy) and the training environment using novel marbles. \textbf{C)} Illustration of the testing condition with marbles in yellow never seen before by neither the ensemble nor the top-policy. \textbf{D)} Mean and standard deviation of the ensemble's output when exposed to novel yellow marbles (test condition). As can be seen the ensemble is able to capture the mean in a reasonable way (the more green marbles the higher the mean Q-values). Importantly, as shown in the second triangle  the variance increases as the number of yellow marbles increases. \textbf{E)} Choice behavior of the combined agent (meta-policy + ensemble). We see that when the top-policy has been exposed to negative reward for blue marbles it behaves in an ambiguity-averse manner whereas the contrary is observed for positive-reward blue marbles. }
\label{fig:ambiguity_one_step}
\end{figure*}

\subsection{Results Experiment 3: Ambiguity sensitivity with described probabilities}

Figure~\ref{fig:ambiguity_one_step}D shows the mean values of $\mathbf Q_\mu$ (left plot) and $\mathbf Q_\sigma$ (right plot) for an ensemble of $K=20$ agents trained under Table~\ref{tab:experiments} conditions. We expose the ensemble to different test configurations involving yellow marbles---never seen before---as measured by the axis on the bottom of the triangle. As can be seen, growing number of yellow marbles implies a growing standard deviation over Q-values provided by the ensemble. This shows that the ensemble is able to detect novel situations in a continuous fashion. Additionally, the left plot shows how the mean Q-values can correctly recognize situations with high number of red or green marbles by assigning high or low mean values, respectively.

Figure~\ref{fig:ambiguity_one_step}E shows the average choice behavior of the combined ensemble and meta-policy for different rewards associated with the blue marbles. In particular, we trained the meta-policy with blue-marble rewards $r_b = \{-1.0, 0.0, 1.0, 2.0 \}$, which are effectively the rewards of stimuli $\mathcal{S}_\text{amb-seen}$ from previous sections. As can be seen,  for negative blue rewards the agent is ambiguity-averse, since it chooses the certain urn more often for increasing amount of yellow marbles. The opposite effect is observed for positive blue rewards: the agent chooses the ambiguous urn more often as the amount of yellow marbles increases. This type of behavior clearly corresponds to ambiguity-averse and ambiguity-seeking agents respectively.

\begin{figure*}
\centering
\includegraphics[width=\textwidth]{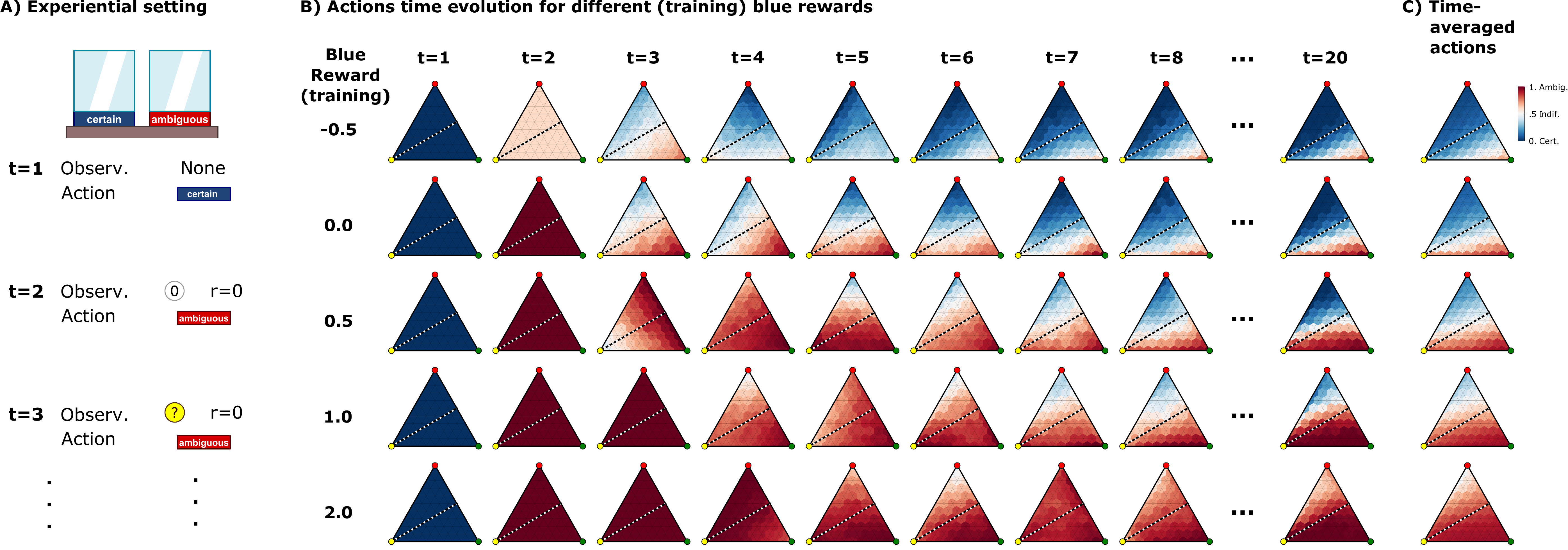}
\caption{\textbf{Experiment 4: Ambiguity-sensitivity with experiential probabilities.} \textbf{A)} Illustration of the sequential setting. At $t=1$ the agent receives an empty observation denoting zero knowledge about the urn contents. As time advances and choices are made the agent observes the sampled marbles from the selected urn. \textbf{B)} Choice behavior for the agents trained with Algorithm~\ref{alg:ambiguity} on conditions of Table~\ref{tab:experiments}. Rows denote different agents trained with different rewards for blue marbles and columns denote the choice behavior at each time-step. As can be seen higher reward for blue marbles generates ambiguity-seeking behavior that chooses more often the ambiguous urn (in red) whereas negative rewards for blue marbles generates ambiguity-averse behavior by choosing more often the certain urn (in blue). \textbf{C)} Time-averaged behavior.}
\label{fig:ambiguity_multi_step}
\end{figure*}

\subsection{Results Experiment 4: Ambiguity sensitivity with experiential probabilities}

Figure~\ref{fig:ambiguity_multi_step}A shows how the experiential setting looks like (see Table~\ref{tab:experiments} for more information). This is very similar to Figure~\ref{fig:risk_multi_step}A, but now the agent can also observe yellow marbles never seen before.

In Figure~\ref{fig:ambiguity_multi_step}B we show the choice behavior for various blue-marble reward training conditions along time-steps. At the first time-step the agent always chooses the certain non-ambiguous urn. This is a similar strategy adopted by the agents at Experiment 2 in Figure~\ref{fig:risk_multi_step}. Similarly here, it is a smart strategy to  know the contents of such urn early on. 
In the next time-steps the differently trained agents behave differently. When the top-policy is trained with negative blue-marble rewards (first row) the agent learns to be ambiguity-averse when observing novel yellow marbles. This can be checked by seeing that most of the choices correspond to selecting the ambiguous urn. Such choices are denoted in blue inside the triangle plot.
When we increase the blue reward during the training condition (see Table~\ref{tab:experiments}), the agent becomes more ambiguity-seeking,  which is denoted by the fact that it chooses the ambiguous urn (in red color code) more often.
Figure~\ref{fig:ambiguity_multi_step}C shows the time-average of the average choice behavior which depicts the same conclusions outlined above since it is just a compressed version of the all the time-step data.

\subsection{Results Experiment 5: Ambiguity sensitivity in a grid-world}
Figure~\ref{fig:ambiguity_grid_world} shows bar plots of the average fraction of green, red, and gray marbles (never seen before) picked up during 80-step episodes starting from $200$ different initial marble configurations. For the top three bars, the top-level policy was trained on blue marbles with positive reward, for the bottom three bars the blue marbles had negative reward. As can be seen, the agent on the top picks up more of the novel gray marbles, thus exhibiting ambiguity-seeking behavior, while still avoiding red marbles and picking up green marbles. The agent on the bottom picks up fewer of the gray marbles indicating that it is ambiguity-averse. Furthermore, it exhibits slightly worse performance on the green marbles as avoiding both red and gray marbles makes it harder to collect all the green ones.

The top-level agent detects ambiguity not only using the variance of the ensemble $Q$-values, but also the mean. If the mean deviates from values observed during training, this is taken as a sign of ambiguity by the top-level policy.

\begin{figure*}
\centering
\raisebox{0.2\height}{\includegraphics[width=0.3\textwidth]{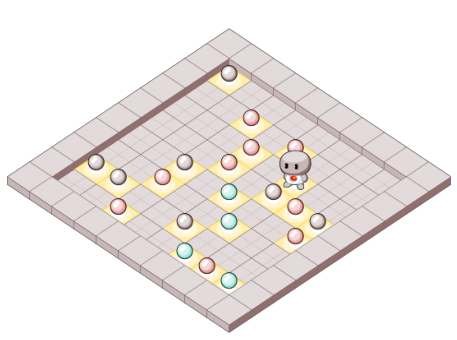}}
\hspace{0.02\textwidth}
\includegraphics[width=0.6\textwidth]{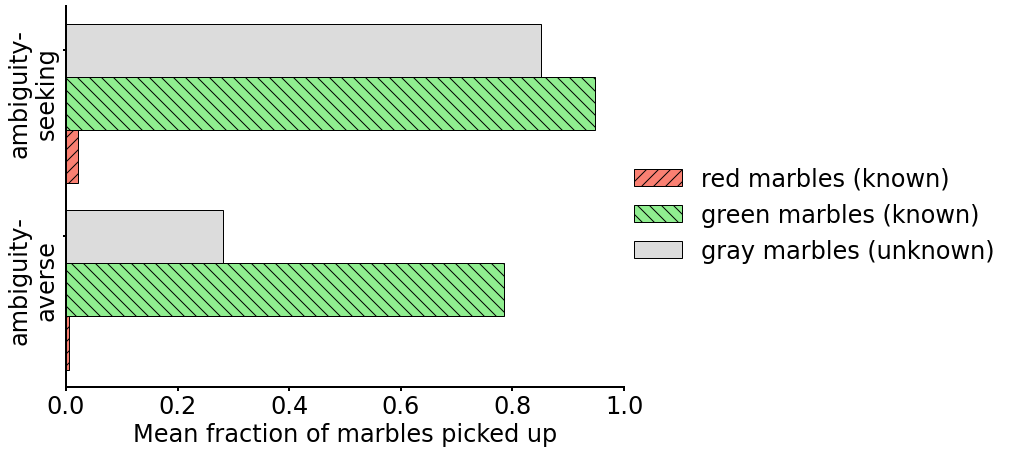}
\caption{\textit{Left:} Grid world environment for testing ambiguity sensitivity. Green (positive reward) and red (negative reward) marbles were encountered during training, gray marbles are introduced only at test time. \textit{Right}: Mean fraction of marbles picked up by agents with positive (top bars) and negative (bottom bars) experience of ambiguity during training. The mean was calculated over episodes with $200$ different random initial marble configurations.}
\label{fig:ambiguity_grid_world}
\end{figure*}

\section{Discussion}\label{sec:discussion}
In this paper, we have proposed two modifications to the meta-learning protocol that produce risk- and ambiguity-sensitive behavior. For risk-sensitivity to emerge, the environment must adapt to the value estimations of the agent in a friendly or adversarial way. For ambiguity-sensitivity to emerge, the agent needs a way to detect novelty for example using an ensemble as proposed in \cite{lakshminarayanan2017simple}. We have shown that given this information, a second meta-policy can be trained to exhibit ambiguity-sensitivity based on its experience in novel situations. We have demonstrated how to practically train and test such agents' uncertainty-sensitivity in urn and grid world experiments. 

As mentioned in Section \ref{sec:basic-risk-sensitive}, an agent trained with our risk-sensitive training protocol can be interpreted as being embodied, that is, contained within the environment it is interacting with. This is because the transition dynamics of the environment now depend on the agents behavior, specifically they depend on it's (estimated) Q-values. Intuition suggests that an agent that is embodied would be more sensitive to risk than an agent that is not. This is because an agent that is embodied in an environment can be affected (and more importantly altered) directly by the environment, while an agent that is separate from the environment cannot be (directly) altered by the environment (or itself).

What is the rational thing to do in absence of full information? The point raised in~\cite{gilboa2009always} highlights that there might not be a perfectly rational decision under ambiguity. The issue comes from the tension between having to make a decision and the inability to justify such decisions under the absence of data, logic or sufficient scientific knowledge. This is resolved in two ways. Either the decision-maker uses default policies when such situations occur or it resorts to incorporating caution (i.e. maxmin rules) into the decision procedure. Our architecture can handle both cases since the meta-policy is able to either learn a default policy, or to incorporate worst-case assumptions e.g. as in the case of negative reward for the blue-marbles. 

But what is ambiguity really? We hypothesize its nature might be closely related to the concepts of open and closed worlds. In closed worlds, the assumption is that all relevant hypothesis (e.g. all causal models) are taken into account in the Bayesian framework. Since the agent can place priors over the hypothesis space, it can also modify its decision-making processes to take into account such uncertainties in a risk-sensitive fashion. In contrast, the open-world assumption contemplates the possibility of a non-complete hypothesis space  (e.g. that there are missing mechanisms in the Bayesian model). Thus, in open worlds, the absence of evidence about the truth of a particular statement doesn't directly mean that the statement is false as in closed-worlds, but that it is simply not known whether it is false or true. This is where the distinction between both can hint to the true concept of "I don't know" without relying on an infinite regression of uncertain priors. In summary, the agent can think about a finite amount of finitely complex world-models. To be aware that this class of models might not be enough to describe the situation at hand, particularly in novel situations, is to be aware of ambiguity. However, as more data is available, these missing mechanisms can potentially be learned, consequently reducing ambiguity to risk.

\section{Conclusion}
Currently, most machine learning models cannot distinguish between what they know and what they don't know. However, this ability is crucial for robust systems that can deal with our highly uncertain and dynamic world.  While robustness is traditionally linked with risk-sensitivity it is also deeply connected to ambiguity. We have shown how to deal with risk and ambiguity in a data-dependent manner with our meta-training mechanisms.
We hope these mechanisms are a starting point of a plethora of data-dependent methods for the study and application of uncertainty-sensitivity in humans and  machines.

\bibliographystyle{unsrt}
\bibliography{main}

\section*{Appendix}

\appendix

\section{Risk-sensitivity on a single-step decision-making problem}\label{appendix:risk-theory-single-step}

\textbf{Setting:} Let us assume that there is a single state $s$, so that our analysis is simpler.  There are actions $a \in \mathcal A$ and observations $o \in \mathcal O$. Let the agent implement a \emph{fixed} distribution over actions $a \sim \pi(\cdot)$, and let the observations depend on the agent's action and be stochastic via the distribution $o \sim P_{O}(\cdot |a)$. Given an observation $o$ the reward is also stochastic with $r \sim P_{R}(\cdot |o)$. The average reward is defined as
\begin{equation}
  V(o):= \sum_r rP_{R}(r|o).
\end{equation}

The learning procedure goes as follows. At iteration $k$, an action is sampled $a_k \sim \pi$, then an observation $o_k \sim P_O(\cdot |a_k)$ and a reward $r_k \sim P_R(\cdot | o_k)$. We gather all the data up to iteration $k$ with the ordered set $\mathcal D^k := \{ (a_i, o_i, r_i) \}_{i=1}^k$.
Now, instead of TD-learning (applicable to the multi-step case), assume that the agent learns  an estimate of $V(o)$ at iteration $k$ via plain Monte Carlo (MC), 
\begin{equation}
   \hat V_k(o) := \frac{1}{|\mathcal I(o, k)|} \sum_{i \in \mathcal I(o, k)} r_i,
\end{equation}
where $\mathcal I(o, k)$ is the set of all indices that satisfy $o == o_i$ in $[\mathcal D^k]_i$. That is, all the rewards that were obtained from observation $o$ up to iteration $k$.

Now we have two possibilities:
\begin{itemize}
    \item  First, the environment could be fixed to $P_O(\cdot | a)$
    \item Second, we could employ a modification $\rho_O (\cdot |a, \hat V)$ that depends on the agent's value estimate.
\end{itemize}

\paragraph{First case:} Assume the first case where the environment is fixed. In this case, the  agent learns an estimate of the value of action  $a$ at iteration $k$ via
\begin{equation}
  \hat Q_k(a)  := \sum_o P_O(o|a) \hat V_k(o).
\end{equation}
Note we could also do the MC approximation of $\hat Q$ but we decide not to do so now for the sake of simplicity. If we take the limit
\begin{equation}
    \lim_{k \rightarrow \infty} \hat Q_k (a) = \sum_o P_O(o|a) V(o)
\end{equation}
we can see that the agent would converge to a risk-neutral valuation.

\paragraph{Second case:} Now assume the second case where the environment adapts to the current value of the agent in the following way
\begin{equation}
    \hat Q_k (a):= \sum_o \rho (o |a, \hat V_k) \hat V_k(o)
\end{equation}
where 
\begin{equation}
   \rho (o|a, \hat V_k) :=\frac{1}{Z} P_O(o |a ) e^{\beta \hat V_k(o)}.
\end{equation}
Using these quantities we have the following proposition.

\newtheorem*{proposition1}{Proposition 1}

\begin{proposition1}
Let $\hat V_k(o)$ be the current valuation that the agent assigns to $o$. Further assume that the environment is modified to be $\rho(o|a,\hat V_k)$ from Equation~\eqref{eq:simple_modified_transition}. Then the Q-values from Equation~\eqref{eq:value-risk-simple}  are risk-sensitive in the sense that they are a function of the expectation and the variance of $V_k$ under $T$ depending on $\beta$. That is
\begin{equation}
    \hat Q(a) \approx \mathbb E_{T(\cdot |a)}[ \hat V(o) ] +  \beta \mathbb{VAR}_{T(\cdot |a)} [\hat V(o)].
\end{equation}
\end{proposition1}

\begin{proof}
The first order Taylor approximation of $\hat Q_k$ at $\beta = 0$ is 
\begin{equation}
    \hat Q_k(a) \approx \mathbb E_{P_O(\cdot |a)}[ \hat V_k(o) ] + \beta  \frac{\partial \hat Q_k(a)}{ \partial \beta }\bigg\rvert_{\beta = 0} 
\end{equation}
where 
\begin{align}
  \frac{\partial \hat Q_k(a)}{\partial \beta} &= \frac{\partial}{\partial \beta} \sum_o \frac{1}{Z} P_O(o|a)e^{\beta \hat V_k(o)} \hat V_k(o) \\
   & = \sum_o  P_O(o|a) \hat V_k(o) \left[  \frac{1}{Z}e^{\beta \hat V_k(o)}  \hat V_k(o) - e^{\beta \hat V_k(o)} \frac{1}{Z^2} \frac{\partial Z}{\partial \beta}  \right]  \\
   & = \sum_o  P_O(o|a) \hat V_k(o) \left[  \frac{1}{Z}e^{\beta \hat V_k(o)}  \hat V_k(o) - e^{\beta \hat V_k(o)} \frac{1}{Z^2} \sum_o P_O(o|a) e^{\beta \hat V_k(o)}\hat V_k(o) \right]  \\
   & = \sum_o \rho(o|a, \hat V_k (o)) \hat V_k(o)^2 - \left( \sum_o \rho(o|a, \hat V_k (o)) \hat V_k(o) \right)^2 \\
   & = \mathbb{VAR}_\rho \left[ \hat V_k(o)\right]
\end{align}
Therefore,
\begin{equation}
    \hat Q_k(a) \approx \mathbb E_{P_O(\cdot |a)}[ \hat V_k(o) ] +  \beta \underbrace{\mathbb{VAR}_\rho [\hat V_k(o)]\rvert_{\beta = 0}}_{\mathbb{VAR}_{P_O(\cdot |a)} [\hat V_k(o)]}
\end{equation}

As can be seen the agent is risk-sensitive at each iteration. Of course, it will also happen in the limit $\lim_{k \rightarrow \infty} \hat Q_k$.
\end{proof}

\section{Risk-sensitivity theory for the sequential case}\label{appendix:risk_theory}
The aim of this section is to show that our methodology in Section~\ref{sec:summary_methodology_risk} generates risk-sensitive agents.

A Markov decision process (MDP) is defined as the tuple ($\mathcal S, \mathcal A, T, r, \gamma)$ where $\mathcal S$ is the state space, $\mathcal A$ is the action space, $T: \mathcal S \times \mathcal A \rightarrow P(\mathcal S)$ the transition function and $r: \mathcal S \times \mathcal A \rightarrow \mathbb R$ the reward function. 
For a particular state $s$ and policy $\pi$, the values are computed as 
\begin{align}
V^\pi(s) & \coloneqq \mathbb E \left[  \sum_{t=0}^\infty \gamma^t r(s_t, a_t) | s_0 = s\right].
\end{align}

Now we will show that the following definition of value generates the same transition dynamics as we use in our experiments. The key insight is to introduce a variational distribution $\psi$ (playing the role of $\rho$ in the main manuscript) and optimize it to maximize the rewards while being close to the standard dynamics $T$.   

\begin{definition}[Risk-sensitive value]
\begin{equation}\label{eq:risk-sensitive_value} 
    V^{\pi\psi} (s) := \mathbb E \bigg[ \lim_{H \rightarrow \infty } \sum_{t=0}^H \gamma^t \bigg( r(s_t, a_t)  - \frac{1}{\beta} \log \frac{\psi(s_{t+1} | s_t a_t) }{T(s_{t+1} | s_t a_t)} \bigg) \bigg| s_0 = s \bigg]
\end{equation}
where $\beta >0$ modulates the strength of the regularizer, and the expectation is over trajectories $\tau$ with $p(\tau |s_0) = \prod_t \pi(a_t|s_t)\psi(s_{t+1}|s_t,a_t)$. 
\end{definition}
The value function $\mathcal V^\pi (s) :=  \max_{\psi} V^{\pi \psi} (s) $  assigns value by penalizing variability on the return due to the stochasticity in the dynamics of the environment. Similar value functions have been adopted in~\cite{fei2020risksensitive}.

\begin{proposition}[Recursion]\label{prop:recursion}
The value function in \eqref{eq:risk-sensitive_value} satisfies the following recursion
\begin{equation}
    V^{\pi \psi} (s)  = \mathbb E_{a \sim \pi(\cdot | s)} \bigg[ r(s, a) +  \sum_{s'} \psi (s' | s, a) \bigg[  \gamma V^{\pi\psi}(s')  - 
    \frac{1}{\beta} \log \frac{\psi (s' |s, a)}{ T(s' |s, a)} \bigg]\bigg].
\end{equation}
\end{proposition}
\label{proof:recursion}
\begin{proof}
\begin{align*}
   V^{\pi\psi} (s) & = \mathbb E \bigg[ \lim_{H \rightarrow \infty } \gamma^0 \left( r(s_0, a_0) -  \frac{1}{\beta} \log \frac{\psi(s_{1} | s_0, a_0) }{T(s_{1} | s_0, a_0)} \right)  \\
   & \qquad + \sum_{t=1}^{H} \gamma^{t} \left( r(s_{t}, a_{t}) - \frac{1}{\beta} \log \frac{\psi(s_{t+1} | s_{t}, a_{t}) }{T(s_{t+1} |s_{t}, a_{t})} \right) \bigg| s_0 = s \bigg]. \\
   & = \mathbb E \bigg[  r(s_0, a_0) -  \frac{1}{\beta} \log \frac{\psi(s_{1} | s_0, a_0) }{T(s_{1} | s_0, a_0)}  \\
   & \qquad + \lim_{H \rightarrow \infty } \sum_{t=0}^{H-1} \gamma^{t+1} \left( r(s_{t+1}, a_{t+1}) - \frac{1}{\beta} \log \frac{\psi(s_{t+2} | s_{t+1}, a_{t+1}) }{T(s_{t+2} |s_{t+1}, a_{t+1})} \right) \bigg| s_0 = s\bigg]. \\
   & = \sum_a \pi(a | s) \sum_{s'} \psi (s' | s, a) \Bigg[  r(s, a) -  \frac{1}{\beta} \log \frac{\psi(s' | s, a) }{T(s' | s, a)}  \\
   & \qquad + \gamma  \mathbb E \bigg[ \lim_{H \rightarrow \infty } \sum_{t=0}^{H-1} \gamma^{t} \left( r(s_{t+1}, a_{t+1}) - \frac{1}{\beta} \log \frac{\psi(s_{t+2} | s_{t+1}, a_{t+1}) }{T(s_{t+2} |s_{t+1}, a_{t+1})} \right) \bigg| s_1 = s'\bigg] \Bigg]. \\
   & = \sum_a \pi(a | s) \sum_{s'} \psi (s' | s, a) \bigg[  r(s, a) -  \frac{1}{\beta} \log \frac{\psi(s' | s, a) }{T(s' | s, a)} + \gamma V^{\pi\psi}(s') \bigg]  \\
\end{align*}
\end{proof}

\begin{proposition}[Optimal Value and Argument] \label{prop:optimal_value} \label{prop:optimal_argument}
The optimal value function when maximizing over $\psi$ is
\begin{equation}
    \mathcal V^{\pi} (s):= \max_\psi V^{\pi \psi} (s)   = \mathbb E_{a \sim \pi(\cdot | s)} \bigg[ r(s, a) 
    + \frac{1}{\beta} \log \sum_{s'} T(s' | s, a) \exp \big( \gamma\beta \mathcal V^{\pi}(s') \big) \bigg]
\end{equation}
with optimal argument
\begin{equation}
\psi^* (s'| s, a) := \frac{T(s'|s,a) e^{\gamma \beta \mathcal V^\pi (s')}}{\sum_{\tilde s}T(\tilde s|s,a) e^{\gamma \beta \mathcal V^\pi (\tilde s)}}.
\end{equation}
\end{proposition}

\begin{proof}
We follow similar proof techniques from~\cite{tishby2011information,rubin2012trading,grau2016planning}. Starting with the equation for $V^{\pi \psi} (s)$
\begin{equation} 
V^{\pi \psi} (s)  = \mathbb E_{a \sim \pi(\cdot | s)} \bigg[ r(s, a) +  \sum_{s'} \psi (s' | s, a) \bigg[  \gamma V^{\pi\psi}(s')  - 
    \frac{1}{\beta} \log \frac{\psi (s' |s, a)}{ T(s' |s, a)} \bigg]\bigg].
\end{equation}
To find the maximum over $\psi$ we solve $\frac{\delta}{\delta \psi} V^{\pi\psi} (s)=0$ for $\psi$. First we need to find $\frac{\delta}{\delta \psi} V^{\pi\psi} (s)$.

\begin{align*}
	&\frac{\delta}{\delta \psi} V^{\pi\psi} (s)\\
	 &= 	\frac{\delta}{\delta \psi} \left[  \sum_a \pi(a | s)  \bigg[ r(s, a) +  \sum_{s'} \psi (s' | s, a) \bigg[  \gamma V^{\pi\psi}(s')  - 
    \frac{1}{\beta} \log \frac{\psi (s' |s, a)}{ T(s' |s, a)} \bigg]\bigg]\right] \\
     &\stackrel{(a)}{=} 	  \sum_a \pi(a | s)  \bigg[ \sum_{s'} \frac{\delta}{\delta \psi} \bigg[  \gamma \psi (s' | s, a) V^{\pi\psi}(s')  - 
    \frac{1}{\beta} \psi (s' | s, a) \log \frac{\psi (s' |s, a)}{ T(s' |s, a)} \bigg]\bigg] \\
    &\stackrel{(b)}{=} 	  \sum_a \pi(a | s)  \bigg[ \sum_{s'}  \bigg[  \gamma \left(  V^{\pi\psi}(s') + \psi (s' | s, a) \left(\frac{\delta}{\delta \psi}V^{\pi\psi}(s')\right) \right) - 
    \frac{1}{\beta} \left( \log \frac{\psi (s' |s, a)}{ T(s' |s, a)} + \psi (s' | s, a) \frac{1}{\psi (s' |s, a)} \right) \bigg]\bigg] \\
    &\stackrel{}{=} 	  \sum_a \pi(a | s)  \bigg[ \sum_{s'}  \bigg[  \gamma \left(  V^{\pi\psi}(s') + \psi (s' | s, a) \left(\frac{\delta}{\delta \psi}V^{\pi\psi}(s')\right) \right) - 
    \frac{1}{\beta} \left( \log \frac{\psi (s' |s, a)}{ T(s' |s, a)} + 1 \right) \bigg]\bigg]
\end{align*}
(a) comes from the derivative being linear and $r(s,a)$ having no dependence on $\psi$. (b) is the use of the product rule of differentiation.

Solving $\frac{\delta}{\delta \psi} V^{\pi\psi} (s)=0$ for $\psi$ we get
\begin{align*}
 \gamma \left(  V^{\pi\psi}(s') \right) &= 
    \frac{1}{\beta} \left( \log \frac{\psi (s' |s, a)}{ T(s' |s, a)} + 1 \right) \\
    \Rightarrow \log \frac{\psi (s' |s, a)}{ T(s' |s, a)}  &= \beta \gamma V^{\pi\psi}(s')  - 1 \\
    \Rightarrow \frac{\psi (s' |s, a)}{ T(s' |s, a)}  &= e^{\beta \gamma V^{\pi\psi}(s')  - 1} \\
    \Rightarrow \psi (s' |s, a) &= T(s' |s, a) e^{\beta \gamma V^{\pi\psi}(s')  - 1} \\
    &= \frac{T(s' |s, a) e^{\beta \gamma V^{\pi\psi}(s')}}{e}
\end{align*}
Normalizing we get
\[ \psi^* (s' |s, a) = \frac{\frac{T(s' |s, a) e^{\beta \gamma V^{\pi\psi}(s')}}{e}}{\sum_{\tilde s} \frac{T(\tilde s |s, a) e^{\beta \gamma V^{\pi\psi}(\tilde s)}}{e} 
} = \frac{T(s' |s, a) e^{\beta \gamma V^{\pi\psi}(s')}}{\sum_{\tilde s} T(\tilde s |s, a) e^{\beta \gamma V^{\pi\psi}(\tilde s)}} \]
Plugging this back into the equation we get
\begin{align*}
	V^{\pi } (s)  &\stackrel{}{=} \mathbb E_{a \sim \pi(\cdot | s)} \bigg[ r(s, a) +  \sum_{s'} \psi^* (s' | s, a) \bigg[  \gamma V^{\pi}(s')  - 
    \frac{1}{\beta} \log \frac{\psi^* (s' |s, a)}{ T(s' |s, a)} \bigg]\bigg] \\
    &\stackrel{(a)}{=} \mathbb E_{a \sim \pi(\cdot | s)} \bigg[ r(s, a) +  \sum_{s'} \frac{T(s' |s, a) e^{\beta \gamma V^{\pi}(s')}}{\sum_{\tilde s} T(\tilde s |s, a) e^{\beta \gamma V^{\pi}(\tilde s)}} \bigg[  \gamma V^{\pi}(s')  - 
    \frac{1}{\beta} \log \frac{\frac{T(s' |s, a) e^{\beta \gamma V^{\pi}(s')}}{\sum_{\tilde s} T(\tilde s |s, a) e^{\beta \gamma V^{\pi}(\tilde s)}}}{ T(s' |s, a)} \bigg]\bigg] \\
    &\stackrel{(b)}{=} \mathbb E_{a \sim \pi(\cdot | s)} \bigg[ r(s, a) +  \sum_{s'} \frac{T(s' |s, a) e^{\beta \gamma V^{\pi}(s')}}{\sum_{\tilde s} T(\tilde s |s, a) e^{\beta \gamma V^{\pi}(\tilde s)}} \bigg[  \gamma V^{\pi}(s')  - 
    \frac{1}{\beta} \log \frac{ e^{\beta \gamma V^{\pi}(s')}}{\sum_{\tilde s} T(\tilde s |s, a) e^{\beta \gamma V^{\pi}(\tilde s)}} \bigg]\bigg] \\
    &\stackrel{(c)}{=} \mathbb E_{a \sim \pi(\cdot | s)} \bigg[ r(s, a) +  \sum_{s'} \frac{T(s' |s, a) e^{\beta \gamma V^{\pi}(s')}}{\sum_{\tilde s} T(\tilde s |s, a) e^{\beta \gamma V^{\pi}(\tilde s)}} \bigg[  \gamma V^{\pi}(s')  - 
    \frac{1}{\beta} \left( \beta \gamma V^{\pi}(s') - \log \sum_{\tilde s} T(\tilde s |s, a) e^{\beta \gamma V^{\pi}(\tilde s)} \right) \bigg]\bigg] \\
    &\stackrel{(d)}{=} \mathbb E_{a \sim \pi(\cdot | s)} \bigg[ r(s, a) + \frac{1}{\beta}  \log \sum_{\tilde s} T(\tilde s |s, a) e^{\beta \gamma V^{\pi}(\tilde s)} + \sum_{s'} \frac{T(s' |s, a) e^{\beta \gamma V^{\pi}(s')}}{\sum_{\tilde s} T(\tilde s |s, a) e^{\beta \gamma V^{\pi}(\tilde s)}} \bigg[  \gamma V^{\pi}(s')  - 
    \frac{1}{\beta}  \beta \gamma V^{\pi}(s')\bigg]\bigg] \\
     &\stackrel{(e)}{=} \mathbb E_{a \sim \pi(\cdot | s)} \bigg[ r(s, a) + \frac{1}{\beta}  \log \sum_{\tilde s} T(\tilde s |s, a) e^{\beta \gamma V^{\pi}(\tilde s)} \bigg]
\end{align*}
(a) is plugging in $\psi^*$. (b) and (c) are just algebra. (d) comes from the term $\frac{1}{\beta}  \log \sum_{\tilde s} T(\tilde s |s, a) e^{\beta \gamma V^{\pi}(\tilde s)}$ not depending on $s'$. (e) comes from $\gamma V^{\pi}(s')  - 
    \frac{1}{\beta}  \beta \gamma V^{\pi}(s')=0$.

\end{proof}

\begin{definition}[Q-values] Using Proposition~\ref{prop:optimal_value} we can define the modified Q-values as
\begin{equation}
    \mathcal Q^{\pi} (s, a):=  r(s, a) + \frac{1}{\beta} \log \sum_{s'} T(s' | s, a) \exp \left(\gamma\beta \mathcal V^{\pi}(s') \right) 
\end{equation}
\end{definition}

Importantly, it can clearly be seen that the second term in the Q-values is a free energy or moment-generating function. As such they capture all the higher-order moments of the return and thus, agents acting with these Q-values are risk-sensitive where their risk-sensitivity is controlled by $\beta$. Positive $\beta$ generates risk-seeking behavior whereas for negative $\beta$ we obtain risk-averse behavior. The limit of $\lim_{\beta\rightarrow0}$ recovers the standard risk-neutral valuation which corresponds to the expectation. For a particular example see~\cite{fei2020risksensitive}.

A caveat is that our agents don't exactly learn these Q-values as we show next in a simplified one-time step case. Consider single step case biased dynamics: 
\begin{equation}\label{eq:biased_variational_observation_model}
   q^*_{\beta}(o|s, a)=\frac{1}{Z(s, a)}p(o|s, a) e^{\beta u(s, a, o)}
\end{equation}
where $Z(s, a)=\sum_o p(o|s, a) e^{\beta u(s, a, o)}$. Its associated free energy is $Q_F (s, a)= \frac{1}{\beta} \log \sum_o p(o|s, a) e^{\beta u(s, a, o)}$.  Then, under our training scheme, the agent computes the following value
\begin{equation}
   Q_{q^*_\beta} (s, a):= \sum_o q^*_\beta (o|s, a) u(s, a, o).
\end{equation}
The crucial point lies in the fact that the functions  $\beta \mapsto Q_F $ and $\beta \mapsto Q_{q^*_\beta}$ are very similar. That is  $Q_{q^*_\beta} = \frac{\partial}{\partial \beta} \log Z= \frac{\partial }{\partial \beta} \beta Q_F = Q_F + \beta \frac{\partial}{\partial \beta } Q_F$. So $Q_{q^*_\beta}$ is equal to  $Q_F$ plus an added term.  The main point here is that since it depends on  $Q_F$, and $Q_F$ is a risk-sensitive valuation (as shown for the multi-step case) the valuation $Q_{q^*_\beta}$ is also risk-sensitive. Additionally, one can easily conclude that both functions are identical at $\beta \in \{ \infty, -\infty, 0\}$. We can (mostly) conclude that $Q_F \approx Q_q$.

\section{Training Protocol Details for Risk-sensitivity}\label{appendix:protocol_details_risk}
Since the state-space can be continuous, $\rho$ can be difficult to compute. However, we do not need to compute $\rho$ exactly, we just need to sample states $x' \sim \rho$. These can be obtained by the following approximate procedure also depicted in Figure~\ref{fig:risk_one_step}C. First, we sample $N$ times from the natural distribution over environments $\tilde B = \{\tilde x_i\}^N_{i=1} \sim T(\cdot |x, a)$. Then, we use these samples to build  a proxy distribution $p_c$ which serves as an approximation to $\rho$. Lastly, we can easily sample the next state from $p_c$.

The proxy distribution is described in more detail now. Let $C$ denote the set of unique elements in $\tilde B$. Now, let $n_j:= \sum_i \mathbb{I} _{\tilde x_i = x_j}$ to be the number times the state $x_j \in C$ has been sampled, where $j$ is just a dummy index for the elements in the set $C$.  For uncountable $\mathcal X$, all samples are unique i.e., $n_j = 1$ for all $j$, however, for countable $\mathcal X$, $n_j$ can be greater than $1$. Note also that, $N=\sum_j n_j$. Next, we construct the vector $\boldsymbol c$ where $c_i := e^{\beta V_\omega (x'_i)}$. By normalizing $\boldsymbol c$ we obtain a probability distribution over $x'$ which we can sample from:
\begin{equation}\label{eq:dynamics_approx}
p_{\boldsymbol c} (x' |x, a):= \frac{n_j e^{\beta V_\theta (x_j)}}{\sum_k n_k e^{\beta V_\theta (x_k)}}  \quad \text{with} \quad x_j = x'
\end{equation}
since the state $x_j$ has been sampled $n_j$ times. Note that for $\beta = 0$ we recover the original dynamics i.e. $p_{\boldsymbol{c}} (x' | x, a) = T(x' | x, a)$. For $\beta \neq 0$ and $N\rightarrow \infty$ we have the following proposition.
\begin{proposition}[Convergence]\label{prop:convergence_to_rho}
We can recover $\rho$ in the limit of infinite samples. That is $\rho(x' | x,a) = \lim_{N \rightarrow \infty} p_{\boldsymbol c} (x')$. 
\end{proposition}
\begin{proof}
Multiply and divide Equation~\eqref{eq:dynamics_approx} by $N$. 
\begin{equation}
 \frac{n_j/N e^{\beta V_\theta (s_j)}}{\sum_j n_j/N e^{\beta V_\theta (s_j)}} 
\end{equation}
Note that the the limit $\lim_{N\rightarrow \infty} n_j/N = T(s'| s, a)$. Therefore, the limit  of the numerator is 
\begin{equation*}
  \lim_{N\rightarrow \infty} n_j/N e^{\beta V_\theta (s')}= T(s' |s, a) e^{\beta V_\theta (s')},
\end{equation*}
and similarly for the denominator.
\end{proof}
 
Now we can easily sample the next state from $x \sim p_c$ with the guarantee that for large $N$ we are approximating $\rho$ due to Proposition \ref{prop:convergence_to_rho}. In our experiments we use $N=10$ and the agent computes the Q-values with $\mathcal Q^{\pi p_{\boldsymbol c}} $ and the policy is $\pi(x) = \arg \max_a \mathcal Q^{\pi p_{\boldsymbol c}}(x, a) $.

\section{Hyper-parameters}\label{appendix:hyperparameters}

The following Table~\ref{tab:hyperparameters} describes the hyper-parameters that we used for training our agents. 

\begin{table}[h]
\begin{center}
\begin{tabular}{>{\raggedleft\arraybackslash}p{4.5cm}p{3.0cm}p{2.0cm}} 
\textbf{Hyper-parameter} & \textbf{Value} \\ \toprule
Discount factor & 0.95   \\ \addlinespace[1.5pt]
Batch size & 128         \\ \addlinespace[1.5pt]
Max learner steps & 1M   \\ \addlinespace[1.5pt]
Replay capacity & 1e5    \\ \addlinespace[1.5pt]
$K$ Q-networks & 20      \\ \addlinespace[1.5pt]
RNN torso, LSTM and head widths & (128, 128, 128)  \\ \addlinespace[1.5pt]
Learning rate & 1e-4     \\ \addlinespace[1.5pt]
Max grad norm & 1.       \\  \addlinespace[1.5pt]
Replay period & 40       \\ \addlinespace[1.5pt]
Min replay size & 500    \\ \addlinespace[1.5pt]
Burn in length & 0       \\ \addlinespace[1.5pt]
\bottomrule 
\end{tabular}
\end{center}
\caption{Hyper-parameters used for training. See~\cite{kapturowski2018recurrent} for more details}
\label{tab:hyperparameters}
\end{table}

\end{document}